\documentclass[11pt]{article}
\def\shownotes{1}

\usepackage[final]{acl}

\usepackage{natbib}
\usepackage{times}
\usepackage{latexsym}

\usepackage[T1]{fontenc}

\usepackage[utf8]{inputenc}

\usepackage{inconsolata}
\usepackage{macro}
\usepackage{graphicx}
\usepackage{amsmath}
\usepackage{times}
\usepackage{latexsym}
\usepackage{makecell}
\usepackage{balance}

\usepackage{caption}
\usepackage{subcaption}
\usepackage{booktabs}
\usepackage{algorithm,algorithmicx,algpseudocode}
\usepackage{enumitem}
\usepackage{multirow}
\makeatletter
\def\thanksnosymbol#1{\protected@xdef\@thanks{\@thanks
        \protect\footnotetext{#1}}}
\makeatother
\usepackage{hyperref}
\newcommand{\acronym}{{\sc{MoPLEx}}}

\title{MoPLEx: Estimating Plackett-Luce Mixture Models for Multi-Objective Alignment}

\author{Dongyue Li\\
Northeastern University, Boston, MA\\
\texttt{li.dongyu@northeastern.edu}\And 
Ziniu Zhang\\
Northeastern University, Boston, MA\\
\texttt{zhang.zini@northeastern.edu} \AND
Lu Wang\\
University of Michigan, Ann Arbor, MI\\
\texttt{wangluxy@umich.edu} \And 
Hongyang R. Zhang\\
Northeastern University, Boston, MA\\
\texttt{ho.zhang@northeastern.edu}
}

\begin{document}
\maketitle

\begin{abstract}
We study learning a mixture of $k$ Plackett-Luce models from multi-way ranking responses from annotators that may represent heterogeneous underlying preferences. This problem has many applications in AI alignment and preference optimization. Prior work has studied mixtures of Bradley–Terry models from pairwise comparisons. However, estimating a mixture of multi-way ranking models can become theoretically unidentifiable when $k$ exceeds $m/2$, where $m$ is the ranking length. We design an efficient algorithm to address this issue by first augmenting the rankings to a larger size (e.g., generating comparisons from a base model), followed by a gradient-based estimation to reduce inference cost (in the input embedding space). With this procedure in mind, we then fit a mixture of Plackett-Luce (PL) models via an expectation-maximization-style iteration, or \acronym{} in short. We conduct extensive experiments to verify this algorithm. First, we find that the gradient-based approximation estimates true probabilities with less than 5\% error on models with up to 34 billion parameters. Second, \acronym{} improves clustering and ranking accuracy by an average of 43.7\% and 15.2\% over baselines using a single PL model or a mixture of Bradley-Terry models, on UltraFeedback and PERSONA datasets. These results demonstrate the effectiveness of \acronym{} for tackling multi-way rankings following heterogeneous preferences through measuring alignment via gradients. 
\end{abstract}

\section{Introduction}

We study language model alignment given preference data drawn from several heterogeneous subpopulations. The central challenge is to identify the clustering structures and learn a corresponding mixture of ranking models \cite{zhao2016learning,awasthi2014learning}. This problem arises in scenarios when preferences are collected from diverse annotator subpopulations or scored along several evaluation criteria \cite{kirk2024prism,castricato2025persona}. These different subgroups can rank candidate responses differently, yielding a mixture of rankings that a single ranking model cannot accurately capture \cite{chakraborty2024maxmin}. In this paper, we tackle this problem by designing a new algorithm to learn a mixture of ranking models using a gradient-based estimation procedure.

One natural approach is to extend the Bradley-Terry (BT) model used in RLHF and DPO~\cite{dpo23} to a mixture model. For example, recent work trains a mixture of BT models by regularizing mixture weights to a uniform prior \cite{shen2025micro}, or by using an expectation-maximization algorithm to alternate between cluster assignment and model optimization \cite{chakraborty2024maxmin}. However, these methods apply to rankings of two choices (i.e., pairwise comparisons), leaving learning mixtures on longer rankings open. While BT models can be generalized to Plackett-Luce (PL) models \cite{dpo23}, existing methods focus on fitting single models rather than mixtures.

\begin{figure*}[t!]
    \centering
    \includegraphics[width=0.99\textwidth]{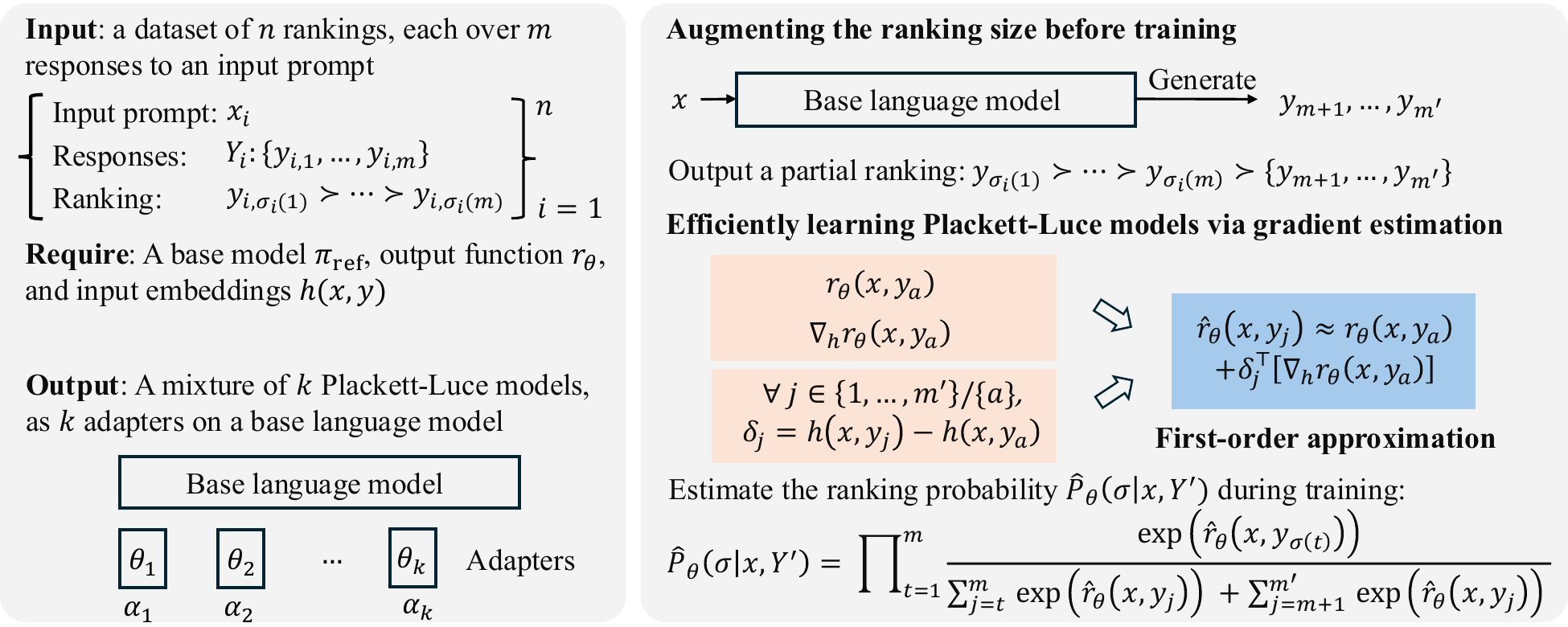}
    \caption{\textbf{Left}: Given $n$ rankings over $m$ candidate responses per prompt, we aim to learn a mixture of $k$ Plackett-Luce models, implemented via an ensemble of $k$ adapters on a base language model. 
    We identify an unidentifiability problem: when the ranking length is small (specifically, when $m \le 2k - 1$), there exist two distinct mixture models that yield the same ranking distribution, making the true clusters unidentifiable. \textbf{Right}: We propose an algorithm to address this problem. First, we augment the ranking length by generating additional responses from the base model, outputting a larger partial ranking.
    Second, to reduce the computational cost, we estimate the model outputs on responses via a first-order approximation in the input embedding space. This computes outputs and gradients for a small number of $a$ anchor responses, reducing the runtime and memory cost from $O(km)$ to $O(ka)$.%
    }\label{fig_overview}
\end{figure*}

A key problem in this setting is identifiability, where two distinct mixture models can produce the same ranking probability distributions, making it hard to identify the true clusters. For example, consider ranking four responses drawn equally from two true clusters: $A \succ B \succ C \succ D$ and $B \succ A \succ D \succ C$. Observing only rankings of length $m \le 3$ yields the same marginal distributions as another clustering ($A \succ B \succ D \succ C$ and $B \succ A \succ C \succ D$), as shown in Figure \ref{fig_unidentifiability}. 
We then formulate a broader pattern that a mixture of $k$ models is unidentifiable when the ranking length $m$ is short. 
Our synthetic experiments show that a mixture of BT models collapses to random guessing when $k>2$, whereas increasing the ranking size beyond $2k$ recovers the true clusters.

This work proposes an algorithm that overcomes this limitation. First, to address insufficient ranking lengths, we augment the ranking length $m$ by generating additional candidate responses using the base language model before training.
Second, training the $k$-mixture ranking models over $m$ candidates requires $O(km)$ forward passes per iteration, which is computationally expensive. We design a gradient-based estimation procedure to reduce the cost to $O(ka)$ with $a$ typically three times smaller than $m$ in practice. Our idea is to first compute exact outputs and gradients for the $a$ anchor responses. Then, we can estimate outputs for the remaining responses using a first-order approximation in the input embedding space, without performing full model computation. 
Taken together, we design an algorithm named \acronym{} to efficiently scale the learning of mixtures of Plackett-Luce models on top of any base model. See Figure \ref{fig_overview} for an illustration of our approach.

We extensively evaluate our approach across a range of preference optimization datasets and language models. First, our gradient estimation method approximates full model outputs with under \textbf{5}\% error on models with up to 34 billion parameters. Second, for datasets involving multiple evaluation criteria and annotator subpopulations, our algorithm improves clustering and ranking accuracy by \textbf{43.7}\% and \textbf{15.2}\%, respectively, compared to baselines that use a single ranking model or a mixture of BT models. Furthermore, our method reduces GPU runtime and memory costs by up to \textbf{3}$\times$ compared to full computation of Plackett-Luce models. By increasing ranking lengths three times, our method yields a \textbf{4.6}\% performance gain over full models on the original short rankings. The code for replicating these empirical findings is available at \href{https://github.com/VirtuosoResearch/MoPLEx-implementation}{https://github.com/VirtuosoResearch/MoPLEx-implementation}.

\section{Preliminaries} 

We study learning a mixture of Plackett-Luce models from a ranking dataset. In the context of text data, we are given rankings over text responses for an input prompt. Our input is a dataset $\mathcal{D} = \{(x_i, Y_i, \sigma_i)\}_{i=1}^n$, where each instance comprises a prompt $x$, a set of $m$ candidate responses $Y = \{y_{1}, \dots, y_{m}\}$, and a multiway ranking $\sigma$. This ranking is a permutation over $[m]$ that defines an order $y_{\sigma(1)} \succ y_{\sigma(2)} \succ \dots \succ y_{\sigma(m)}$, where $\sigma(t)$ denotes the index of the response at rank $t$.

In practice, rankings are often gathered from diverse annotator populations or evaluation criteria. This can yield different rankings over the same candidate responses where a single ranking model cannot adequately capture \cite{chakraborty2024maxmin,shen2025micro}.
We illustrate an example in Figure \ref{fig_illustration_of_ranking_differences}. 
Thus, we model such ranking datasets by a mixture of $k$ Plackett-Luce ranking models, each corresponding to a cluster, where $k$ is treated as a hyperparameter. Our objective is to identify the cluster assignment for each ranking and to learn the corresponding ranking models.

\begin{figure}[t!]
    \centering
    \begin{minipage}[b]{0.24\textwidth}
        \centering
        \includegraphics[width=\textwidth]{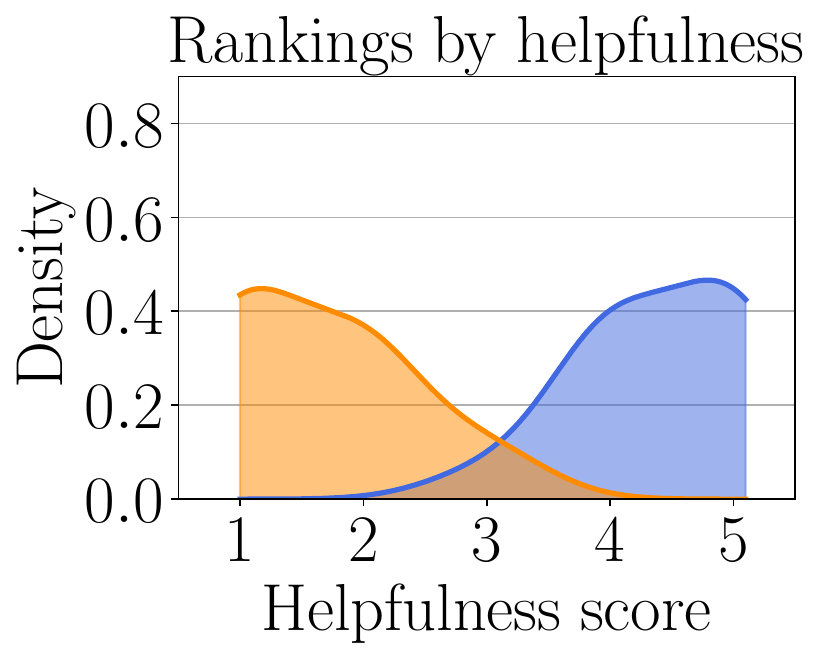}
    \end{minipage}\hfill
    \begin{minipage}[b]{0.24\textwidth}
        \centering
        \includegraphics[width=\textwidth]{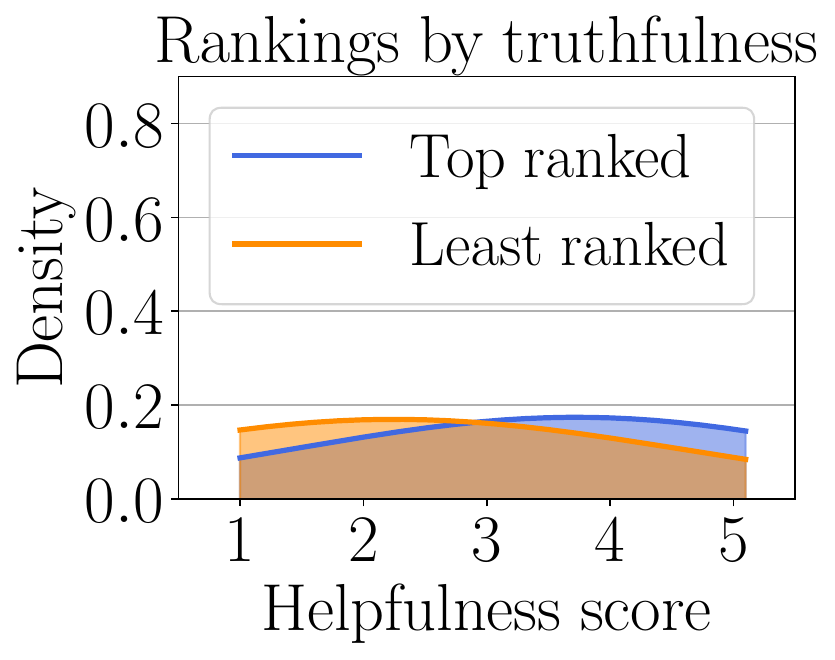}
    \end{minipage}
    \caption{We plot the helpfulness scores for the top and bottom-ranked responses from the UltraFeedback dataset. \textbf{Left}: Responses ranked by helpfulness exhibit clear score separation. \textbf{Right}: Rankings based on truthfulness are nearly indistinguishable using helpfulness scores. Thus, a single ranking model cannot capture rankings given by various evaluation criteria.
    }
    \label{fig_illustration_of_ranking_differences}
\end{figure}

\begin{figure*}[t!]
    \centering
    \begin{minipage}[t!]{0.7\textwidth}
        \centering
        \includegraphics[width=\textwidth]{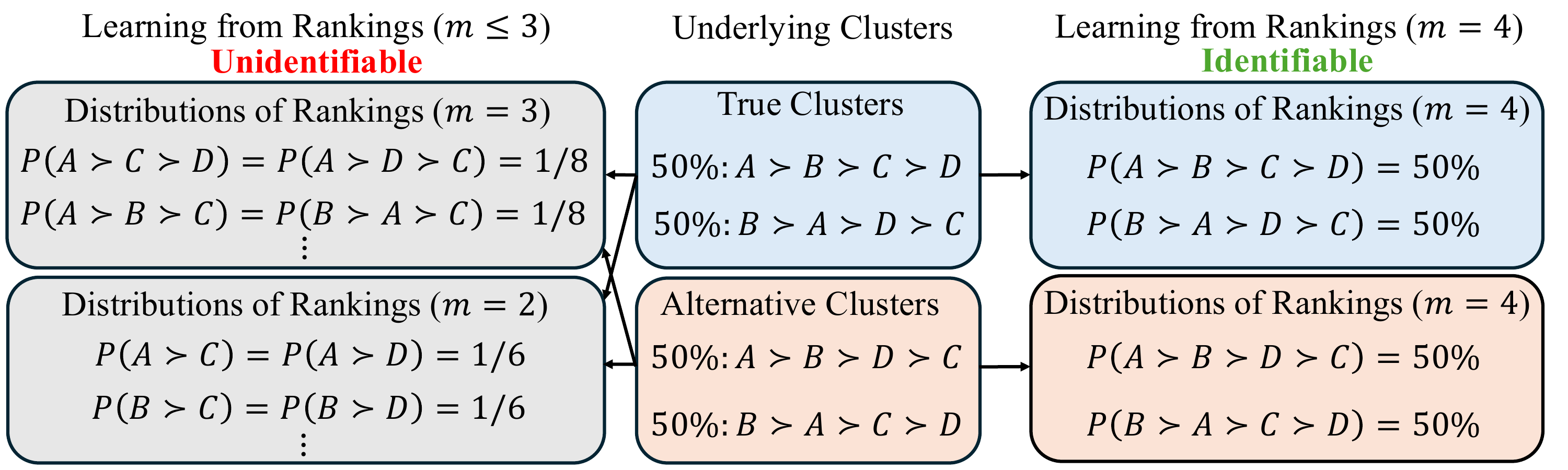}
        \caption{ We illustrate the unidentifiability limit of a mixture model with $k = 2$. \textbf{Left}: If the given dataset contains rankings of length $m = 3$ (or $m = 2$), two possible cases of underlying clusters can yield the same marginal distributions. The mixture model is mathematically incapable of identifying the true clusters. 
    \textbf{Right}: If the given dataset contains rankings of length $m = 4$, the two sets of clusters yield distinct distributions, and the mixture model can identify the correct underlying clusters. 
    }\label{fig_unidentifiability}
    \end{minipage}\hfill
    \begin{minipage}[t!]{0.28\textwidth}
        \centering
        \includegraphics[width=0.9\textwidth]{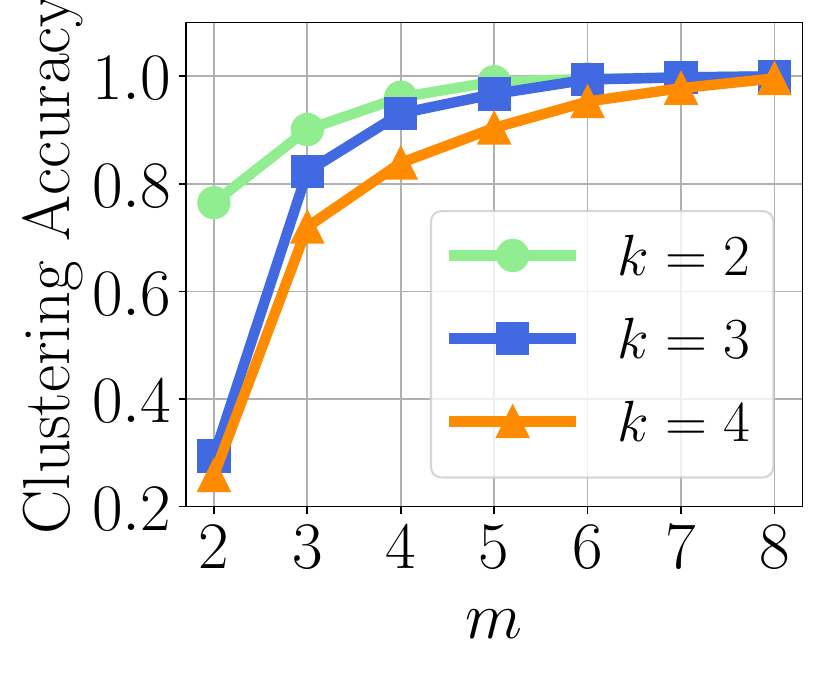}
        \caption{We validate the unidentifiability of the mixture of PL models in a synthetic experiment, with varying numbers of clusters $k$. We find that for a small $m$, the underlying clusters cannot be fully recovered.} \label{fig_ranking_accuracy_varying_m_k}
    \end{minipage}
\end{figure*}

\paragraph{Mixtures of Plackett-Luce models.}  Let $z_i \in \{1, \dots, k\}$ denote the cluster assignment for example $i$, and let $\alpha_c$ be the prior probability of cluster $c$, where $\sum_{c=1}^k \alpha_c = 1$. For a given cluster $c$, the likelihood of observing a specific ranking $\sigma_i$ is given by a ranking model $P_{\theta_c}(\sigma_i \mid x_i, Y_i, z_i=c)$. The objective is to learn the model parameters $\set{\theta_c}^k_{c=1}$ and mixing weights $\set{\alpha_c}^k_{c=1}$ by maximizing the log-likelihood:
\begin{align}\label{eq_mixture_of_ranking}
    \sum_{i=1}^n \log \left( \sum_{c=1}^k \alpha_c \cdot P_{\theta_c}(\sigma_i \mid x_i, Y_i, z_i = c) \right).
\end{align}
One natural approach to estimate the mixture model is via the expectation-maximization algorithm. Each ranking $\sigma_i$ is assigned to the cluster that yields the highest posterior probability.

The Plackett-Luce (PL) model is commonly used to compute the probability of a ranking. 
Let $r_{\theta_c}(x, y)$ be the score assigned by the model for cluster $c$. $P_{\theta_c}(\sigma_i \mid x_i, Y_i, z_i = c)$ is defined as:
\begin{align}\label{eq_mixture_of_PL_prob}
\prod_{t=1}^{m-1} \frac{\exp\left( r_{\theta_c}(x_i, y_{i, \sigma_i(t)}) \right)}{\sum_{l=t}^m \exp\left( r_{\theta_c}(x_i, y_{i, \sigma_i(l)}) \right)}.
\end{align}
Substituting \eqref{eq_mixture_of_PL_prob} into Equation \eqref{eq_mixture_of_ranking} yields the mixture of PL models. 
The mixture of BT models is a special case when $m=2$. Extensions to other ranking models are discussed in Appendix \ref{sec_extension}.

In practice, we can fine-tune a language model to parameterize the score $r_{\theta_c}(x, y)$. For example, following DPO \cite{dpo23}, we can define it as the scaled log-probability ratio:
\begin{align}
r_{\theta_c}(x, y) := \beta \log \frac{\pi_{\theta_c}(y \mid x)}{\pi_{\text{ref}}(y \mid x)},\label{eq_definition_of_score}    
\end{align}
where  $\pi_{\text{ref}}$ is a base reference model and $\beta$ is a hyperparameter controlling the KL-divergence.

\paragraph{Problem statement.} Given an input of $n$ rankings, each over $m$ candidate responses to an input prompt, we aim to fit a mixture of $k$ Plackett-Luce models on top of these rankings (each corresponding to a cluster), with mixture weights of clusters $\{\alpha_c\}_{c=1}^k$ and model parameters of each cluster $\{\theta_c\}_{c=1}^k$. 
In particular, the estimation problem involves 1) identifying the cluster assignment for each ranking, and 2) estimating the corresponding Plackett-Luce (PL) model for each group, and its mixture weight.
Within each PL model, we will assume that the score of each input ranking is computed according to Equations \eqref{eq_mixture_of_PL_prob} and \eqref{eq_definition_of_score}.

\section{Unidentifiability of Mixture Models} \label{sec_unidentifiability}

This section presents a limitation that the mixture of Plackett-Luce (PL) models is unidentifiable when the size $m$ is small. Unidentifiability occurs when two distinct sets of mixtures produce the same ranking distributions, making it impossible to recover the true clusters (See Definition \ref{def_identifiability}). This limits the ability of mixture Bradley-Terry models to identify multiple clusters.

Figure \ref{fig_unidentifiability} illustrates this limitation using four candidate responses ($A, B, C, D$). Suppose the true distribution consists of two equal clusters: $A \succ B \succ C \succ D$ and $B \succ A \succ D \succ C$. If we only observe rankings of length $m \leq 3$, this setup yields the same marginal distributions as an alternative clustering ($A \succ B \succ D \succ C$ and $B \succ A \succ C \succ D$). Thus, it is difficult to distinguish the true clusters from alternative ones. In contrast, learning from rankings of length $m=4$ can identify the clusters correctly.

Building on the insights from \citet{zhao2016learning}, we note that the unidentifiability of the mixture of PL models depends on the cluster number $k$ and the ranking size $m$. 

\begin{proposition}[Slightly adapted from \citet{zhao2016learning}]
\label{prop_non_identifiability_mixtures}
For a mixture of $k$ PL models defined over $m$-way rankings, when $m \le 2k - 1$, the model is non-identifiable. That is, there exist two sets of distinct weights and parameters that induce an identical distribution over the space of $m$-way rankings.
\end{proposition}

The result suggests that, for a larger number of clusters, more candidate responses are needed to identify the underlying clusters. The proof is deferred to Appendix \ref{sec_proof_prop_1}.

We now validate the above limitation of the mixture models. We construct a synthetic dataset to simulate various criteria. Each candidate response is represented by a list of $k$ integers, where the $c$-th integer serves as the score for cluster $c \in \{1, \dots, k\}$. We train a Qwen-3-0.6B model, varying the ranking size $m$ from 2 to 8 and the clusters $k$ from 2 to 4. We then measure the clustering accuracy on a held-out set between the inferred and true cluster assignments. Details are described in Section \ref{sec_additional_exp}. 

As shown in Figure \ref{fig_ranking_accuracy_varying_m_k}, the underlying clusters cannot be fully recovered when $m$ is small. Notably, relying on pairwise comparisons (when $m=2$) collapses to random guessing for $k=3$ and $k=4$. Conversely, increasing the ranking size enables the model to identify the true clusters.

\section{Our Approach}

In this section, we present an efficient algorithm for learning mixture models. Our algorithm involves two components: First, we augment the ranking by generating additional responses. Second, we design an efficient gradient-based implementation for estimating Plackett-Luce mixtures on the expanded rankings.

\subsection{Augmenting the ranking}

First, we propose to augment the ranking by generating additional responses. For an input prompt $x$, we expand the response set to size $m^\prime$ by generating $m^\prime - m$ new responses with the base model. Because no annotation is available for the generated responses, we treat them as ranked lower than existing responses as an approximation.
These new responses are appended as unranked candidates at the bottom of the existing ranking, resulting in a partial $m^\prime$-way ranking: $y_{\sigma(1)} \succ \dots \succ y_{\sigma(m)} \succ \{y_{m+1}, \dots, y_{m^\prime}\}$.

Our motivation is that the augmented responses are
sampled from an unaligned base model, so they are likely to be ranked below the original candidates that are curated by annotators. To validate this, we score the generated responses using a ranking model trained on each evaluation criterion of the UltraFeedback dataset. Over 81\% of the generated responses score below existing responses, and none is ranked in the top two positions. On the other hand, considering all possible cases of ranking all additional responses is computationally expensive. 

Denote $Y^\prime = Y \cup \{y_{m+1}, \dots, y_{m^\prime}\}$. We compute the probability $P_{\theta_c}(\sigma \mid x, Y^\prime, z = c)$ as:
\begin{align}\label{eq_mixture_of_PL_prob_augmented}
\prod_{t=1}^{m} \frac{\exp\left( r_{\theta_c}(x, y_{\sigma(t)}) \right)}{\sum_{l=t}^m \exp\left( r_{\theta_c}(x, y_{\sigma(l)}) \right) + Z^\prime(r_{\theta_c})},
\end{align}
where $Z^\prime(r_{\theta_c}) := \sum_{j=m+1}^{m^\prime} \exp(r_{\theta_c}(x, y_j))$. The scores of the augmented responses are included in the denominator.  
Then, we learn the mixture of PL models by substituting \eqref{eq_mixture_of_PL_prob_augmented} into Equation \eqref{eq_mixture_of_ranking}.  

Next, we show that with augmented responses, the identifiability can be achieved.

\begin{proposition}\label{prop_identifiability_with_augmented_examples}
Let $Y^{\textup{aug}} = \{y_{m+1}, \dots, y_{m^\prime}\}$ be a set of additional responses, forming an expanded set $Y^\prime = Y \cup Y^{\textup{aug}}$ of size $m^\prime$. 
Let $\cS_{m^\prime\!, m}$ be the space of partial rankings where each $\sigma \in \cS_{m^\prime\!, m}$ orders $m$ responses of $Y^\prime$ and leaves the other $m^\prime - m$ at the end of the ranking:
\begin{align*}
y_{\sigma(1)} \succ \dots \succ y_{\sigma(m)} \succ \{ y_{\sigma(m+1)}, \dots, y_{\sigma(m^\prime)} \},
\end{align*}
so that $\vert \cS_{m^\prime\!, m} \vert = m^\prime! / (m^\prime - m)!$, and the probability of each $\sigma \in \cS_{m^\prime\!, m}$ is given by Equation \eqref{eq_mixture_of_PL_prob_augmented}.
Assume that the augmented responses have distinct scores in every cluster $c$, meaning $r_{\theta_c}(x, y_i) \neq r_{\theta_c}(x, y_j)$ (in Equation \eqref{eq_definition_of_score}) for all distinct $y_i, y_j \in Y^{\textup{aug}}$. When $m^\prime - m \ge k$ and $m \geq 2$, the mixture of $k$ PL models over $\cS_{m^\prime\!, m}$ is identifiable in the sense of Definition \ref{def_identifiability}.
\end{proposition}

The idea is that conditioned on the assumption of distinct scores for the generated responses, a matrix representing all possible ranking probabilities under the mixture model has sufficiently many independent rows to match its column dimension, thus exhibiting full column rank. The proof is given in Appendix \ref{sec_proof_prop_2}.

We empirically verify the assumption that generated responses have distinct scores. Recall that the score is computed via model log-likelihoods over the responses. Generating non-duplicate responses at a high sampling temperature can produce responses with distinct scores. We evaluate the scores of the generated responses on the UltraFeedback dataset. We find that using a sampling temperature of $2$ yields no responses with the same scores under each evaluation criterion. Furthermore, we compute the coefficient of variation for their scores, computed as the standard deviation of the scores divided by their mean. The coefficient of variation is over 73\%, indicating substantial variance in the scores. 

\subsection{Estimating Plackett-Luce mixtures}

Learning the mixture model on expanded rankings is computationally expensive. Evaluating a mixture of $k$ PL models over $m$ responses involves $km$ full model forward passes at each iteration. Next, we introduce a method to reduce the computational cost from $O(km)$ to $O(ka)$ with a much smaller constant $a$ than $m$.

We estimate $r_{\theta}(x, y)$ based on a first-order approximation property that has been observed widely on language models \cite{zhang2025linear}.
The idea is to estimate the output of the model using a first-order approximation in the input embedding space. 
Let $h(x, y)$ be the input embedding of the response $y$ for prompt $x$, and let $r_{\theta_c}(x, y)$ be a function of $h(x, y)$. Given an anchor example $(x, y_0)$, we examine the first-order approximation of the model output $r_{\theta}(x, y)$ as:
\begin{align}\label{eq_linear_approximation}
& r_{\theta}(x, y) \approx r_{\theta}(x, y_0) + \notag \\
& \langle \nabla_{h} r_{\theta}(x, y_0), h(x, y) - h(x, y_0) \rangle + \epsilon_{x,y},
\end{align}
where $\nabla_{h}$ denotes the gradient of $r_\theta$ with respect to the input embeddings $h$, and $\epsilon_{x,y}$ is the approximation error. For responses of different lengths, we apply the approximation up to the length of the shorter response and ignore padding tokens.

We demonstrate empirically that the error $\epsilon_{x,y}$ remains small across various models and datasets. We evaluate $\epsilon_{x,y}$ across three LMs from 0.6 to 34 billion parameters, on the UltraFeedback dataset \cite{cui2023ultrafeedback}. For each example $(x, Y, \sigma)$, we randomly select one anchor response and estimate the outputs of the remaining responses. We compute the estimation errors over 50 randomly selected examples. Table \ref{tab_linear_approximation} reports the average squared relative error, $( {\epsilon_{x,y}} / {r_{\theta}(x, y)} )^2$, grouped by the relative distance in their input embeddings, ${\norm{h(x, y) - h(x, y_0)}} / {\norm{h(x, y_0)}}$.

We find that the estimation error remains less than \textbf{5}\%. Additionally, the model with more parameters tends to yield a smaller error. These results suggest that Equation \eqref{eq_linear_approximation} can accurately approximate the model output for estimating ranking probabilities. Similar results are observed on other datasets and are described in detail in Table \ref{tab_linear_approximation_helpsteer} of Appendix \ref{sec_additional_exp}.

\begin{table}[t!]
\centering
\caption{Relative approximation error of $\epsilon_{x, y}$, tested with language models of up to 34 billion parameters.}\label{tab_linear_approximation}
\resizebox{\columnwidth}{!}
{\begin{tabular}{c|ccccccc}
\toprule
Distance & Qwen-0.6B & Gemma-2B & CodeLlama-34B \\
\midrule
$0\%-5\%$   & $0.1_{\pm0.0}\%$ & $0.1_{\pm0.1}\%$ & $0.1_{\pm0.0}\%$\\
$5\%-10\%$  & $0.6_{\pm0.1}\%$ & $0.5_{\pm0.1}\%$ & $0.6_{\pm0.0}\%$\\
$10\%-15\%$ & $1.7_{\pm0.2}\%$ & $1.5_{\pm0.1}\%$ & $1.4_{\pm0.1}\%$\\
$15\%-20\%$ & $2.9_{\pm0.1}\%$ & $3.2_{\pm0.3}\%$ & $3.3_{\pm0.3}\%$\\
$20\%-25\%$ & $4.3_{\pm0.3}\%$ & $4.8_{\pm0.4}\%$ & $4.0_{\pm0.3}\%$\\
\bottomrule
\end{tabular}}
\label{tab_approx_error}
\end{table}

We then leverage this approximation to efficiently train the mixture of PL models. Specifically, we apply the first-order approximation to estimate the output for each response in a PL model. At each iteration, we randomly sample a subset of $a$ anchor responses, $Y_a \subseteq Y$. For each anchor $y_i \in Y_a$, we compute the exact output $r_\theta(x, y_i)$ and its gradient $g_i = \nabla_{h} r_\theta(x, y_i)$. For the remaining $m-a$ candidates, we estimate their outputs $\hat{r}_\theta(x, y_j)$ by applying Equation \eqref{eq_linear_approximation} and averaging the first-order approximations derived from all $a$ anchors:
\begin{align*}
\frac{1}{a} \sum_{y_i\in Y_a} \left( r_\theta(x, y_i) + g_i^\top (h(x, y_j) - h(x, y_i)) \right). 
\end{align*}
The procedure is described in Algorithm \ref{alg_prob_estimation}. 

Next, we fit a mixture of PL models with the expectation-maximization algorithm, which alternates between two steps. 
First, given $\theta^{(t)}$ and $\alpha^{(t)}$ at step $t$, we compute the posterior probability that the ranking $\sigma_i$ was generated by cluster $c$:
\begin{align}\label{eq_e_step}
\gamma_{i,c} := \frac{\alpha_c^{(t)} \cdot \hat{P}_{\theta^{(t)}}(\sigma_i \mid x_i, Y_i, z_i = c)}{\sum_{l=1}^k \alpha_l^{(t)} \cdot \hat{P}_{\theta^{(t)}}(\sigma_i \mid x_i, Y_i, z_i = l)}.
\end{align}

Second, we update the parameters by maximizing the log-likelihood. The prior probabilities are updated as: $\alpha_c^{(t+1)} = \frac{1}{n} \sum_{i=1}^n \gamma_{i,c}$. Then, we update model parameters by minimizing the weighted negative log-likelihood as:
\begin{align}\label{eq_m_step}
- \sum_{c=1}^k \sum_{i=1}^n \gamma_{i,c} \log \hat{P}_{\theta^{(t)}_c}(\sigma_i \mid x_i, Y_i, z_i = c).
\end{align}
We do not backpropagate gradients on the outputs estimated by gradients. 

\begin{algorithm}[t!]
\caption{Gradient-based Estimation of Ranking Probabilities}
\label{alg_prob_estimation}
\textbf{Input:} Input prompt $x$, candidate responses $Y = \{y_1, \dots, y_m\}$, ranking $\sigma$, current model $r_\theta$ \\
\textbf{Require:} Number of anchors $a$
\begin{algorithmic}[1]
\State $Y_a \leftarrow$ Sample $a$ anchor responses from $Y$
\For{$y_i \in Y_a$} \hfill $\triangleright$ Exact evaluations
    \State $r_\theta(x, y_i) \leftarrow$ Model output 
    \State $g_i = \nabla_{h} r_\theta(x, y_i) \leftarrow$ Gradient over the input embedding space 
\EndFor
\For{$y_j \in Y \setminus Y_a$}\hfill $\triangleright$ Approximation
    \For{$y_i \in Y_a$}
    \State $ \delta_{i,j} \leftarrow h(x, y_j) - h(x, y_i)$
    \EndFor
    \State $\hat{r}_\theta(x, y_j) \leftarrow \frac{1}{a} \sum_{y_i\in Y_a} ( r_\theta(x, y_i) + {g}_i^\top {\delta}_{i,j} )$ 
\EndFor
\State $\hat{P}_{\theta_c} \leftarrow$ Estimate ranking probability in Equation \eqref{eq_mixture_of_PL_prob_augmented} using the estimated scores
\State \textbf{Return} $\hat{P}_{\theta_c}$
\end{algorithmic}
\end{algorithm}

Taken together, we summarize the complete procedure in Algorithm \ref{alg_em_mixture_pl}. 
We compare our approach to related methods in terms of runtime and memory in Table \ref{tab_computation_efficiency}.
In Appendix \ref{sec_extension}, we will also discuss an extension of our approach to estimating a mixture of Mallows models.

\begin{algorithm}[t!]
\caption{Learning a Mixture of Plackett-Luce Models with Expectation-Maximization~(\acronym{})}
\label{alg_em_mixture_pl}
\textbf{Input:} $n$ rankings $\mathcal{D} = \{(x_i, Y_i, \sigma_i)\}_{i=1}^n$, a base language model $\pi_{\text{ref}}$ \\
\textbf{Require:} Number of clusters $k$, ranking size $m^\prime$, number of anchors $a$, learning rate $\eta$, number of training epochs $T$ 
\begin{algorithmic}[1]
\For{$(x_i, Y_i, \sigma_i) \in \mathcal{D}$}  \hfill \Comment{Augmentation}
    \State $Y_i^{\text{g}} \leftarrow$ Generate $m^\prime-m$ responses by $\pi_{\text{ref}}$
    \State $Y_i^\prime \leftarrow Y_i \cup Y_i^{\text{g}}$
\EndFor
\end{algorithmic}
\begin{algorithmic}[1]
\State ${\theta^{(0)}_c} \leftarrow$ Initialize from $\pi_{\text{ref}}$ for $c \in \{1, \dots, k\}$
\State $\alpha^{(0)}_c \leftarrow 1/k$ for $c \in \{1, \dots, k\}$
\For{$t = 0, \dots, T-1$}
    \For{$(x_i, Y^\prime_i, \sigma_i) \in \mathcal{D}$}\hfill \Comment{E-step}
            \State $\hat{P}_{\theta^{(t)}_c} \leftarrow$ Apply Algorithm \ref{alg_prob_estimation} to estimate ranking probability for $c \in \{1, \dots, k\}$
        \State $\gamma_{i,c} \leftarrow$ Compute the posterior probability in Equation \eqref{eq_e_step} for $c \in \{1, \dots, k\}$ 
    \EndFor
    \For{$c \in \{1, \dots, k\}$}\hfill \Comment{M-step}
        \State $\alpha^{(t+1)}_c \leftarrow \frac{1}{n} \sum_{i=1}^n \gamma_{i,c}$ 
        \State $\hat{P}_{\theta^{(t)}_c} \leftarrow$ Apply Algorithm \ref{alg_prob_estimation} %
        \State $\theta^{(t+1)}_c \leftarrow $ Minimize Equation \eqref{eq_m_step}  
    \EndFor
\EndFor
\State \textbf{Return} $\{{\theta^{(T)}_c}, \alpha^{(T)}_c\}_{c=1}^k$
\end{algorithmic}
\end{algorithm}

\section{Experiments}\label{sec_experiments}

We evaluate our algorithm by answering the following questions: How accurately does our algorithm identify underlying clusters and rank responses, particularly compared to baselines using mixtures of BT models? To what extent does our algorithm reduce computational and memory overhead? How does each component of our algorithm affect downstream performance?

We evaluate our methods across datasets involving multiple evaluation criteria and diverse user demographics. Our method outperforms baselines using single ranking models and mixtures of BT models by \textbf{43.7}\% and \textbf{15.2}\% in clustering and ranking accuracy, respectively. Furthermore, it reduces computational and memory costs by up to \textbf{3}$\times$, matching the performance of training PL models with full computation. Lastly, ablation studies isolate each component of the algorithm.

\subsection{Experimental setup}

\begin{table}[t!]
\caption{A summary of the runtime and memory usage by \acronym{}. Here, $k$ is the number of mixture components, $m$ is the number of responses in rankings, and $a$ is the number of anchor responses. $T$ is the runtime of the base training method, and $A$ is the memory usage on one response. In practice, the runtime and memory overhead ( $T_g$ and $A_g$) in the gradient estimation
is much smaller than $T$ and $A$. $a$ is typically three times smaller than $m$.}
\label{tab_computation_efficiency}
\centering
{\small
\begin{tabular}{lccc}
\toprule
\textbf{Approach} & \textbf{Runtime} & \textbf{Memory} \\ \midrule
Fit a single PL model & $mT$ & $mA$ \\
Fit mixtures of PL models & $kmT$ & $kmA$  \\
\acronym{} (Algorithm \ref{alg_em_mixture_pl}) & $kaT + T_g$ & $kaA + A_g$ \\
\bottomrule
\end{tabular}}
\end{table}

\noindent\textit{Datasets and models.}  First, we consider settings with multiple evaluation criteria, each as a distinct cluster. We use the UltraFeedback dataset \cite{cui2023ultrafeedback}, which ranks four candidate responses ($m=4$) per prompt across four criteria: helpfulness, honesty, instruction-following, and truthfulness ($k=4$). We use a subset where these criteria follow different rankings, yielding 3,612 training, 448 validation, and 452 test rankings.

Second, we consider settings where rankings originate from diverse user demographics, treating each as a cluster. We use the Persona dataset \cite{castricato2025persona}, which provides pairwise comparisons ($m=2$). We select a subset of 12 distinct profiles ($k=12$) spanning two age ranges, two sexes, and four races, resulting in 1,905 training, 237 validation, and 258 test rankings. We use Qwen-3-0.6B as the base language model. For both datasets, the input prompts have no overlap between training and test. %

\begin{table}[t!]
\centering
\caption{We report the average clustering and ranking accuracy (\%) on the UltraFeedback and Persona datasets. We compare \acronym{} against baselines that train single ranking models and mixtures of BT models. We run each experiment with three random seeds and report the average results with standard deviations.}
\label{tab_main_results_summary}
{\small
\begin{tabular}{lcc}
\toprule
{Methods} & {UltraFeedback} & {Persona} \\ \midrule
\multicolumn{3}{l}{{Clustering Accuracy}} \\ \midrule
MiCRo & 26.6 $\pm$ 2.1 & 19.6 $\pm$ 0.4 \\
MaxMin-RLHF & 27.2 $\pm$ 0.4 & 18.4 $\pm$ 0.2 \\
EM-DPO & 26.9 $\pm$ 2.1 & 18.3 $\pm$ 0.5 \\
\midrule
\acronym{} (Alg. \ref{alg_em_mixture_pl}) & \textbf{70.9} $\pm$ 2.3 & \textbf{57.1} $\pm$ 0.5 \\
\midrule\midrule
\multicolumn{3}{l}{{Ranking Accuracy}} \\ \midrule
DPO & 58.2 $\pm$ 1.5 & 57.7 $\pm$ 0.2 \\
LiPO & 63.3 $\pm$ 0.6 & 57.1 $\pm$ 0.8 \\
MiCRo & 57.7 $\pm$ 0.4 & 53.2 $\pm$ 0.8 \\
MaxMin-RLHF & 60.6 $\pm$ 0.7 & 49.8 $\pm$ 0.1 \\
EM-DPO & 59.3 $\pm$ 0.9 & 49.6 $\pm$ 1.5 \\
\midrule
\acronym{} (Alg. \ref{alg_em_mixture_pl}) & \textbf{75.0} $\pm$ 2.9 & \textbf{76.4} $\pm$ 0.6 \\
\bottomrule
\end{tabular}}
\end{table}

\begin{figure}[t!]
    \centering
    \begin{minipage}[b]{0.24\textwidth}
        \centering
        \includegraphics[width=\textwidth]{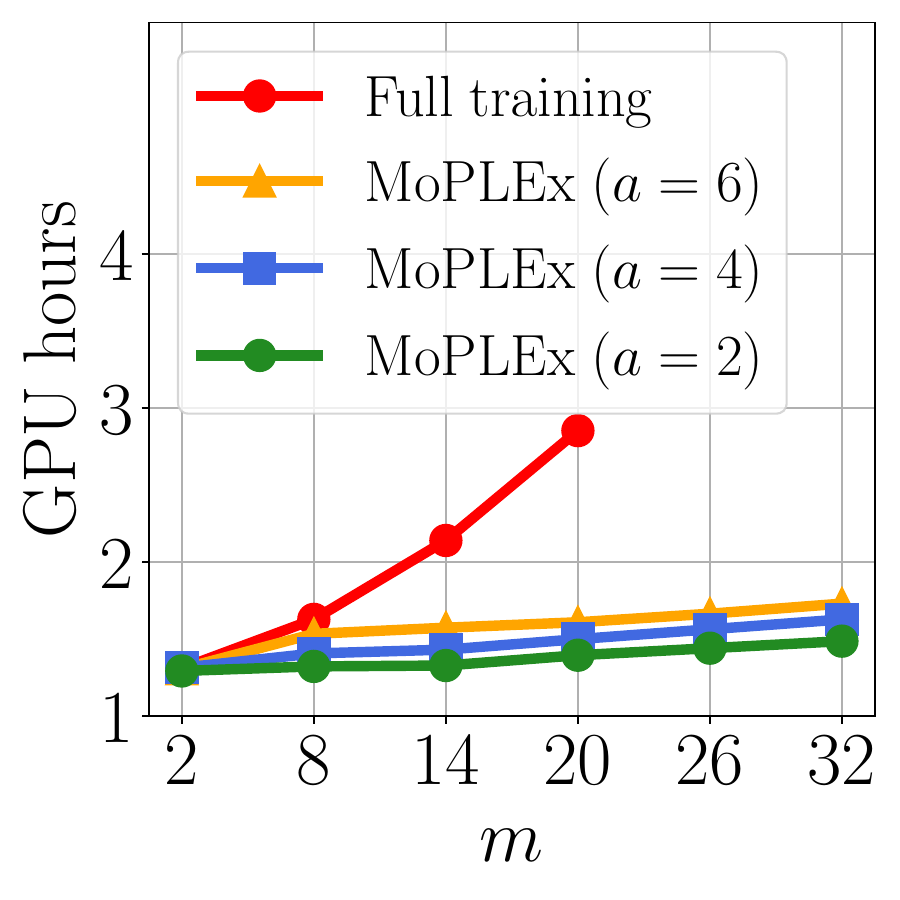}
    \end{minipage}\hfill
    \begin{minipage}[b]{0.24\textwidth}
        \centering
        \includegraphics[width=\textwidth]{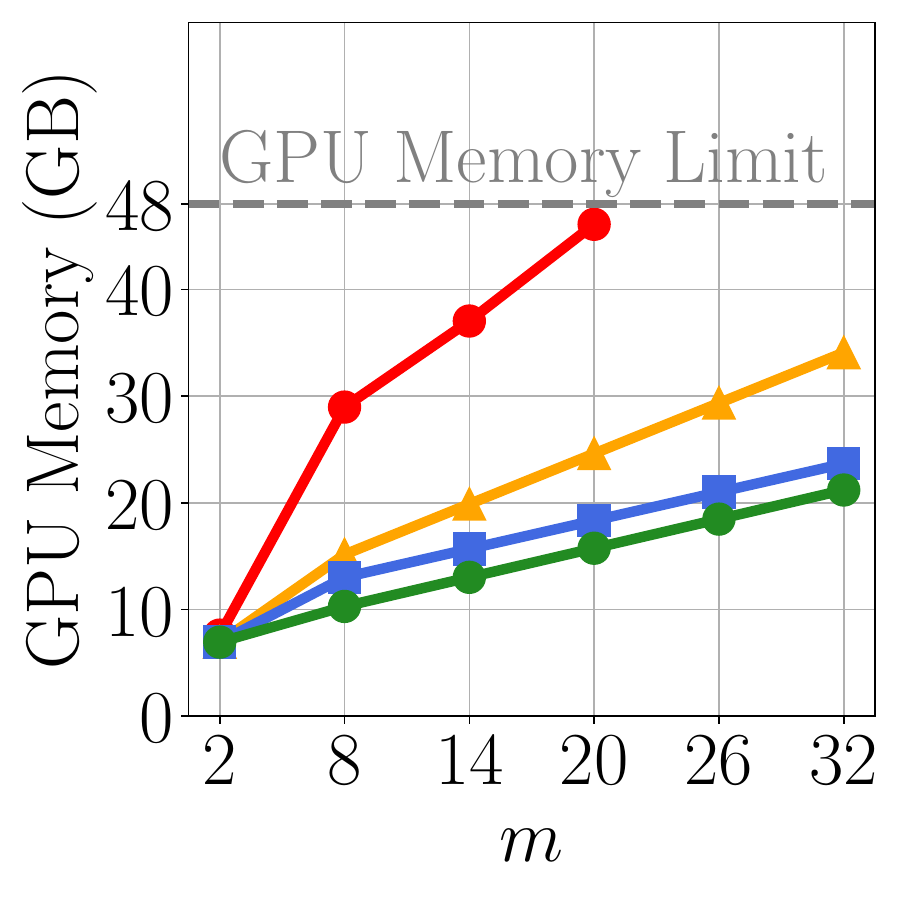}
    \end{minipage}
    \caption{We compare GPU hours and GPU memory between our method and fully training the mixtures of PL models, evaluated on the Persona dataset with $k=12$ clusters. Our method reduces runtime and memory cost by up to \textbf{3}$\times$ with comparable performance. 
    We evaluate the training by using a batch size of 2 and a context length of 1024. } 
    \label{fig_runtime_and_memory}
\end{figure}

\smallskip
\noindent\textit{Baselines.} We compare against two types of representative baselines: (1) single ranking models, DPO~\cite{dpo23} and its extensions; and (2) mixtures of BT models (MaxMin-RLHF~\cite{chakraborty2024maxmin}, MiCRo~\cite{shen2025micro}, and EM-DPO~\cite{chidambaram2025direct}). We fix the base model for all baselines. 

\smallskip
\noindent\textit{Implementations.} We use DPO as the base method and fine-tune with LoRA adapters. We set $k$ to $4$ and $12$ for the datasets, respectively. We set the ranking sizes $m^\prime$ between $4$ and $32$ and the number of anchors $a$ between $1$ and $8$. The effects of varying $k$, $m^\prime$, and $a$ are presented in Section \ref{sec_ablation}.

\subsection{Experimental results}

\begin{table*}[t!]
\centering
\caption{We report the ranking accuracies (\%) on the UltraFeedback dataset across individual evaluation dimensions and their average. We compare our approach against baselines using a single ranking model and mixtures of BT models. 
In addition, we report an ablation study of leaving out each component of \acronym{} on the UltraFeedback dataset.
Ranking accuracy is measured as the proportion of correctly predicted pairs out of the six possible pairs per four-response ranking. We report the mean and standard deviation over three random seeds.}
\label{tab_ablation_leave_on_out}
{\small
\begin{tabular}{lccccc}
\toprule
 Ranking accuracy (\%) & {Helpfulness} & {Honesty} & {Instruction-following} & {Truthfulness} & {Average} \\\midrule
DPO  & 56.4 $\pm$ 3.5 & 53.1 $\pm$ 0.2 & 60.0 $\pm$ 1.6 & 72.2 $\pm$ 0.0 & 58.2 $\pm$ 1.5 \\
LiPO & 63.8 $\pm$ 0.9 & 58.8 $\pm$ 0.9 & 56.2 $\pm$ 0.5 & 74.5 $\pm$ 0.4 & 63.3 $\pm$ 0.6 \\
MiCRo & 58.8 $\pm$ 1.4 & 52.1 $\pm$ 2.9 & 57.2 $\pm$ 3.9 & 70.8 $\pm$ 2.0 & 57.7 $\pm$ 0.4 \\
MaxMin-RLHF & 73.0 $\pm$ 0.6 & 51.1 $\pm$ 0.1 & 57.8 $\pm$ 1.6 & 48.6 $\pm$ 1.9 & 60.6 $\pm$ 0.7 \\
EM-DPO & 65.7 $\pm$ 8.3 & 54.2 $\pm$ 7.4 & 58.3 $\pm$ 0.8 & 61.9 $\pm$ 2.8 & 59.3 $\pm$ 0.9 \\
\midrule
\acronym{} (Algorithm \ref{alg_em_mixture_pl})  & 79.4 $\pm$ 3.9 & 73.0 $\pm$ 1.3 & 76.9 $\pm$ 2.6 & 70.5 $\pm$ 1.3 & 75.0 $\pm$ 2.9 \\
w/o Plackett-Luce models & $75.7 \pm 2.1$ & $71.3 \pm 3.7$ & $70.4 \pm 2.3$ & $69.2 \pm 2.3$ & $71.6 \pm 2.7$ \\
w/o generated responses & $75.0 \pm 2.9$ & $67.2 \pm 2.2$ & $72.2 \pm 2.6$ & $68.3 \pm 2.7$ & $70.7 \pm 2.6$ \\
w/o mixture models & $71.2 \pm 2.8$ & $63.8 \pm 1.4$ & $66.4 \pm 0.8$ & $60.6 \pm 2.1$ & $65.5 \pm 1.8$ \\
w/o gradient estimation & $80.4 \pm 2.7$ & $72.2 \pm 1.7$ & $76.8 \pm 2.8$ & $72.8 \pm 2.7$ & $75.6 \pm 2.3$ \\
\bottomrule
\end{tabular}
}
\end{table*}

\textit{Evaluation metrics.} We evaluate the mixture models on both clustering and ranking accuracy. For clustering, we assign each ranking to the cluster with the highest posterior probability. Due to permutation, we align inferred labels with ground-truth clusters via optimal matching to compute accuracy. For ranking, we measure how accurately each model predicts the rankings of every pair of responses within its assigned cluster, reporting the average across all clusters.

Table \ref{tab_main_results_summary} summarizes the performance on the UltraFeedback and Persona datasets. \acronym{} yields a \textbf{43.7}\% improvement in clustering accuracy over mixture-of-BT baselines (including MiCRo, MaxMin-RLHF, and EM-DPO). 
Furthermore, \acronym{} achieves a \textbf{15.2}\% relative improvement in ranking accuracy over these baselines on average. Results for each cluster on the UltraFeedback dataset are described in Table \ref{tab_ablation_leave_on_out}, and the rest are described in Appendix \ref{sec_additional_exp}.

We further evaluate the efficiency by measuring the training runtime in GPU hours and GPU memory usage in GB.
Figure \ref{fig_runtime_and_memory} compares our method against a mixture of PL models with full computation on the Persona dataset, varying the ranking size $m$ from 2 to 32 and the number of anchors $a$ between $2, 4,$ and $6$. Notably, fully training the PL model exceeds the 48GB memory limit whenever $m > 20$. By using $a = 2$ anchors, our method maintains comparable performance while reducing runtime and memory overhead by 2$\times$ and 3$\times$, respectively. Furthermore, with $a = 6$ anchors, our approach scales to rankings of $m = 32$ and outperforms the best fully trained PL models (at $m = 20$) by 4.6\%, while using 1.6$\times$ less runtime and 1.9$\times$ less memory.

\subsection{Ablation studies}\label{sec_ablation}

\noindent\textit{Effect of each component in \acronym{}.}
First, we conduct an ablation study to isolate the effect of each component in our algorithm. We compare our approach against four variants: (1) replacing Plackett-Luce models with Bradley-Terry models by transforming multi-way rankings into pairwise comparisons (w/o PL models); (2) training the PL mixture model without generating new responses (w/o generated responses); (3) training a single PL model instead of a mixture of PL models (w/o mixture models); and (4) using exact computation for all responses instead of gradient estimation (w/o gradient estimation). We report the results in Table \ref{tab_ablation_leave_on_out} on the UltraFeedback dataset. The results show that  our algorithm outperforms the first three variants. Furthermore, it matches the ranking accuracy of training a mixture of PL models with full computation while cutting GPU memory costs by 50\%.

\noindent\textit{Effect of generated responses diversity.} Next, we evaluate how the diversity of generated responses affects our algorithm. To control the diversity, we train our algorithm on the UltraFeedback dataset, varying the sampling temperature during generation between 0.5, 1.0, 2.0, and 4.0. Higher temperature yields more diverse responses. We show the ranking accuracy for the four criteria and their average in Table \ref{tab_ablation_temperature} of Appendix \ref{sec_additional_exp}. The results show that a high sampling temperature (around 2.0) yields diverse responses and benefits ranking accuracy. Higher temperatures degrade the generated responses and reduce accuracy. 

\noindent\textit{Choice of hyperparameters.} We conduct ablation studies to analyze hyperparameters on the Persona dataset. Increasing $k$ beyond 12 yields no significant gains in clustering or ranking accuracy. Second, expanding the augmented ranking size $m^\prime$ from 2 to 32 improves performance. These gains plateau after $m^\prime = 26$. Finally, although more anchors improve approximation, increasing $a$ beyond 6 provides no obvious gain to overall performance. We report the full results in Appendix \ref{sec_additional_exp}.

\section{Related Work}

Efficient learning from ranking data is an important problem in areas like social choice and language model alignment \cite{conitzer2024position}. 
To capture more complex rankings, prior works have modeled intransitive rankings \cite{makhijani2019parametric}, partial rankings \cite{awadelkarim2024statistical}, and ranking distributions centered around multiple orderings \cite{seshadri2020learning}. 
Most relevant to our study is learning mixture models for ranking. Prior work has established identifiability conditions for Plackett-Luce (PL) models and designed a generalized method of moments to learn two-component PL models \cite{zhao2016learning}. Polynomial-time algorithms are developed to provably learn mixtures of two Mallows models via tensor decomposition \cite{awasthi2014learning} and learn mixtures with any constant number of components \cite{liu2018efficiently}. While prior works focus on learning from rankings such as voting data, our work studies learning mixtures of ranking models on text datasets and designs efficient algorithms for language models.

Several studies have designed ranking models beyond BT models in preference optimization. GPO \cite{zhangbeyond} embeds responses into a latent space to capture intransitive pairwise rankings. Recent work has proposed learning a single multi-way ranking model with surrogate losses to approximate ranking metrics \cite{zhao2025permutative}. For datasets with multiple evaluation criteria, prior works have trained regression models on evaluation scores \cite{wang2024interpretable,wang2024arithmetic}.  EM-DPO \cite{chidambaram2025direct} learns mixtures of BT models from partial rankings. In contrast, our work focuses on learning mixtures of multi-way ranking models.

Another line of work involves computing the influence of training data on model predictions. A commonly used technique is influence functions, effective for tracking the impact of training segments in sequence tagging \cite{jain2022influence}. Prior work proposed function vectors to represent input-output functions using a set of attention heads in language models \cite{todd2024function}. To model task relationships, existing methods compute higher-order task affinity aggregated from multitask training outcomes across task subsets \cite{li2023boosting} or fit linear surrogate models to approximate these outcomes \cite{li2023identification}. 
Gradients have been used as features for efficiently computing model influence functions for modeling task relationships \cite{li2024scalable,li2024scalableft,yang2025precise}. Recent methods leverage these gradient features to group tasks and train weighted ensembles of low-rank adapters \cite{li2025efficient}. Such gradient-based approaches have also been extended to model task relationships in multi-objective reinforcement learning \cite{zhang2025scalable} and multitask algorithmic reasoning \cite{li2026efficiently}. %

There is a vast body of work on learning from a mixture of data distributions. %
For example, previous studies have applied kernel methods to learn predictors for a test distribution using labeled samples from several related distributions \cite{blanchard2011generalizing}. In a multitask learning setting for predicting rewards across multiple arms in contextual bandits, kernel mean embeddings have been used to measure task similarities by characterizing each arm based on its contextual probability distribution \cite{deshmukh2017multi}. %
While traditional semantic class learning and label propagation rely on graph structural linkages \cite{kozareva2011class}, our algorithm resolves structural unidentifiability in rank aggregation through generative ranking augmentation.
Mixture proportion estimation has been studied to estimate the proportion with clean and corrupted data, for training classifiers from noisy labels \cite{scott2015rate}.
Expectation-maximization algorithms are designed to jointly fit truncated multivariate Gaussian distributions to truncated and censored data \cite{lee2012algorithms}.  In comparison, our work designs an algorithm to learn a mixture of Plackett-Luce models from a text dataset with multiple preference criteria.

\section{Conclusion}

This paper investigates learning from datasets with multiple distinct ranking distributions. We identify the unidentifiability limit of mixture ranking models. To address this, we propose an efficient algorithm that resolves unidentifiability and scales to large models. Evaluations on datasets with preferences from multiple evaluation criteria or demographic populations demonstrate our method's performance and efficiency advantages over existing approaches, including learning a single ranking model and mixtures of Bradley-Terry models. 

\section*{Limitations and Future Works}

First, our algorithm treats generated responses ranked lower than existing responses in the dataset as an approximation. While empirical evidence supports this approximation, exceptions where generated responses rank higher can occur in practice. One alternative is to compute the probability of all possible rankings for new responses. While computing all possible cases is computationally expensive, designing an efficient algorithm to estimate these rankings is an interesting direction for future work.
Second, while our current approach learns mixtures of ranking models with a fixed number of components, dynamically adapting this number in continual learning settings remains a promising open problem. Additionally, our method evaluates single-turn conversational data. Extending this approach to multi-turn conversations where user preferences dynamically evolve presents another valuable direction for future research. Our method provides a foundation to scale up learning ranking models for such extensions.

\section*{Potential Risks}

We discuss a few potential risks. 
Our work focuses on learning a mixture of Plackett-Luce models to identify clusters of rankings. Because the algorithm groups preferences into distinct clusters, it might learn spurious correlations between the preferences and certain specific demographics. One future direction is to examine misclassification errors in each group and consider group distributionally robust optimization. 

\section*{Acknowledgment}

We are grateful to Byron Wallace for several discussions related to this paper.
Thanks to Zhenshuo Zhang and Youran Ye for their feedback.
The work of D. Li, Z. Zhang, and H. Zhang is in part supported by NSF award IIS-2412008.

AI assistants have been used to enhance the clarity and presentation of the work. Specifically, the tools are used for proofreading, identifying typographical or grammatical mistakes, and polishing writing. Furthermore, AI assistants are used in formatting and verifying the mathematical notations and derivations of this paper.

\bibliography{./bibliography/references,./bibliography/ref_ranking}
\clearpage

\appendix
\section{Omitted Technical Materials}

\subsection{Notations}

We describe the notation used throughout the paper. The dataset $\mathcal{D} = \{(x_i, Y_i, \sigma_i)\}_{i=1}^n$ consists of $n$ ranking instances. Each instance pairs an input prompt $x_i$ with a set $Y_i = \{y_1, \dots, y_m\}$ of $m$ candidate responses and a ranking $\sigma_i$, a permutation of the candidate indices recording the order in which the responses are preferred. 
$m^\prime$ denotes the ranking length once the generated responses have been appended. $a$ denotes how many responses are evaluated exactly as anchors for gradient-based estimation.
$h(x, y)$ is the input embedding of response $y$ under prompt $x$.
 
$k$ is the number of clusters, each corresponding to one ranking distribution.
$z_i$ is the unobserved cluster assignment of instance $i$. The mixing weight $\alpha_c$ is the prior probability of cluster $c$, and $\gamma_{i,c}$ is the corresponding posterior, that is, the responsibility of cluster $c$ for instance $i$. For a cluster $c$, the ranking probability is written as $P_{\theta_c}$ and is determined by a scalar score $r_{\theta_c}(x, y)$ assigned to each response. The score is given by a language model $\pi_{\theta_c}$ and a reference model $\pi_{\text{ref}}$, with $\beta$ controlling the strength of the KL-divergence penalty between them.

\subsection{Preliminaries}

To precisely describe the concepts used in this paper, we first define the identifiability of a mixture of $k$ Plackett-Luce models, following the terminology from \citet{zhao2016learning}. 

A Plackett-Luce (PL) model defines the probability distribution over rankings of $m$ alternatives. In our context, these alternatives are candidate responses $\{y_1, \dots, y_m\}$ to a prompt $x$. A ranking $\sigma$ specifies the order $y_{\sigma(1)} \succ y_{\sigma(2)} \succ \dots \succ y_{\sigma(m)}$, where $\sigma(t)$ denotes the response index at rank $t$. Given a score vector $r$ in a space $\set{r | r_i \in [0, 1], \text{for } i = 1, \dots, m, \text{and } \sum_{i=1}^m r_i =1}$, with a score assigned to each alternative, the probability of observing ranking $\sigma$ is:
\begin{align*}
P(\sigma \vert{} r ) = \prod_{i=1}^{m-1} \frac{r_{\sigma(i)}}{\sum_{j=i}^m r_{\sigma(j)}}.
\end{align*}
Recall that we compute these scores using a language model parameterized by $\theta$, denoted as $r_\theta(x, y)$, using DPO.
In this section, we focus on analyzing the identifiability of PL models in terms of the score vector $r$ following \citet{zhao2016learning}. We do not discuss the identifiability of the network parameters $\theta$, while recognizing that distinct network weights could yield identical scores for the same input.

For a PL model with score vectors $r, v \in \mathbb{R}^m$, the model is strictly identifiable if $P(\cdot\vert{}r) = P(\cdot\vert{}v)$ implies $r = v$. For mixture models, we state the definition to account for label switching, where permuting the indices of the mixture components does not change the marginal distribution.

\begin{definition}[Identifiability of $k$-PL]
\label{def_identifiability}
A mixture of $k$ Plackett-Luce models ($k$-PL) over $m$ alternatives is identifiable if there do not exist two distinct mixture formulations that produce the same distribution over all rankings. Specifically, the model is identifiable if there do not exist:
\begin{enumerate}
    \item Two integers $k_1, k_2 \le k$,
    \item A set of pairwise different score vectors $r^{(1)}, \dots, r^{(k_1)}$ and $v^{(1)}, \dots, v^{(k_2)}$ (no two vectors are exactly the same),
    \item Strictly positive mixing weights ${\alpha}^{(1)} = (\alpha^{(1)}_1, \dots, \alpha^{(1)}_{k_1})$ and ${\alpha}^{(2)} = (\alpha^{(2)}_1, \dots, \alpha^{(2)}_{k_2})$ where $\sum_{r=1}^{k_1} \alpha^{(1)}_r = 1$ and $\sum_{r=1}^{k_2} \alpha^{(2)}_r = 1$,
\end{enumerate}
such that for every possible ranking $\sigma$, the following equality holds:
\begin{align}
\sum_{c=1}^{k_1} \alpha^{(1)}_c P(\sigma \vert{} r^{(c)}) = \sum_{c=1}^{k_2} \alpha^{(2)}_c P(\sigma \vert{} v^{(c)}).
\end{align}
\end{definition}

To prove the identifiability of $k$-PL models, we analyze the rank of a probability matrix involving the probabilities over $m!$ possible rankings over $m$ alternatives.
For a PL model with a score vector $r$, we represent a distribution over all $m!$ rankings as an $m! \times 1$ column vector, denoted as $f_m(r)$, where each entry is the probability of a specific ranking $\sigma$. For example, when $m=3$, we have:
\begin{align*}
    f_3(r) = \begin{pmatrix} \Pr(y_1 \succ y_2 \succ y_3 \vert{} r) \\ \Pr(y_1 \succ y_3 \succ y_2 \vert{} r) \\ \Pr(y_2 \succ y_1 \succ y_3 \vert{} r) \\ \Pr(y_2 \succ y_3 \succ y_1 \vert{} r) \\ \Pr(y_3 \succ y_1 \succ y_2 \vert{} r) \\ \Pr(y_3 \succ y_2 \succ y_1 \vert{} r) \end{pmatrix} = \begin{pmatrix} \frac{r_1 r_2}{1-r_1} \\ \frac{r_1 r_3}{1-r_1} \\ \frac{r_1 r_2}{1-r_2} \\ \frac{r_2 r_3}{1-r_2} \\ \frac{r_1 r_3}{1-r_3} \\ \frac{r_2 r_3}{1-r_3} \end{pmatrix}.
\end{align*}

Then, we define a probability matrix $F$ as an $m! \times 2k$ matrix formed by concatenating the column vectors of these $2k$ PL models. Given $r^{(1)}, \dots, r^{(2k)}$, we define
\begin{align}
F = \Big[ f_m(r^{(1)}), f_m(r^{(2)}), \dots, f_m(r^{(2k)}) \Big].
\end{align}
The entry  $F_{\sigma, c}$ corresponds to the probability of a ranking $\sigma$ in the $c$-th PL model. Note that $F$ is a function of the $2k$ score vectors. 

We prove the identifiability or non-identifiability of $k$-PL by analyzing the rank of $F$. 
If the rank of $F$ is $2k$ for all pairwise different vectors ($r^{(1)}, \dots, r^{(2k)}$), then the $k$-PL model is identifiable. Conversely, if the rank of $F$ is less than $2k$,  the $k$-PL model is non-identifiable.

\subsection{Proof for Proposition \ref{prop_non_identifiability_mixtures}}\label{sec_proof_prop_1}

Now we present the proof of Proposition \ref{prop_non_identifiability_mixtures} for the non-identifiability of mixture ranking models. 
The key idea is to show that the row space of the probability matrix $F$ can exhibit a reduced dimensionality under a construction. Under this construction, the row space of the probabilities $m!$ possible rankings can be represented by $m$ row vectors, thus proving that the rank $F$ is less than $2k$. We then prove that this dimensionality collapse leads to two $k$-PL models that yield the same distributions.

\begin{proof}[Proof of Proposition \ref{prop_non_identifiability_mixtures}] 
Our proof is based on analyzing the $m! \times 2k$ probability matrix $F$. 
To demonstrate unidentifiability, we aim to prove that a non-zero coefficient vector $\beta \in \mathbb{R}^{2k}$ exists such that $F \beta = \mathbf{0}$. 

To prove the unidentifiability, we construct a specific symmetric structure for $r^{(c)}$, and show that under such a construction, the mixture of the ranking models is unidentifiable. We construct a combined pool of $2k$ distinct latent clusters. Let $r^{(c)}$ denote the score vector for cluster $c \in \{1, \dots, 2k\}$. For a given set of $m$ alternatives $Y = \{y_1, y_2, \dots, y_m\}$, we assign the following scores under each cluster $c$:
\begin{itemize}
    \item We assign a distinct score $e_c \in (0, 1)$ to alternative $y_1$. We ensure $e_i \neq e_j$ for all $i \neq j$.
    \item We assign an identical score $b_c = \frac{1 - e_c}{m-1}$ to all remaining alternatives $y_2, \dots, y_m$.
\end{itemize}

Now, we evaluate the entries of $F_{\sigma, c}$ under this specific construction. Let $t \in \{1, \dots, m\}$ denote the rank position of alternative $y_1$ in the permutation $\sigma$. In the Plackett-Luce model, the probability of selecting the alternatives in the order specified by $\sigma$ is computed as a sequence of $m-1$ choices. This sequential product simplifies by grouping the choices into three cases:
\begin{itemize}
    \item First, before rank $t$ (steps $1$ to $t-1$), an alternative with score $b_c$ is chosen. Assuming the scores sum to 1 ($e_c + (m-1)b_c = 1$), the remaining sum of scores at step $j$ is $1 - (j-1)b_c$. The probability for these steps is $\prod_{j=1}^{t-1} \frac{b_c}{1 - (j-1)b_c}$.
    \item Second, at rank $t$, alternative $y_1$ with score $e_c$ is chosen. The remaining sum of scores is $1 - (t-1)b_c$. The probability is $\frac{e_c}{1 - (t-1)b_c}$.
    \item Third, after rank $t$ (steps $t+1$ to $m$), all remaining $m-t$ alternatives have the identical score $b_c$. By symmetry, the probability of any specific ordering of these remaining identical alternatives is simply $\frac{1}{(m-t)!}$.
\end{itemize}
Therefore, the likelihood entry simplifies to a function that depends strictly on $y_1$'s rank position $t$ and the cluster $c$. Let us denote this unique probability value as $v_{t, c}$:
\begin{align}
v_{t, c} = \frac{1}{(m-t)!} \frac{e_c b_c^{t-1}}{\prod_{j=0}^{t-1} (1 - j b_c)}.
\end{align}

Because the probability $F_{\sigma, c}$ evaluates exactly to $v_{t, c}$ for \emph{any} ranking $\sigma$ where $y_1$ is at rank $t$, the $m!$ rows of the likelihood matrix $F$ consist of only $m$ unique types of rows. We can formally express this by factorizing the $m! \times 2k$ matrix $F$ into the product of two smaller matrices, $M$ and $V$:
\begin{equation}
F = M V,
\end{equation}
where we explicitly construct $M$ and $V$ as follows:
\begin{itemize}
    \item $V$ is an $m \times 2k$ matrix capturing the unique probability values across all clusters. Its $(t, c)$-th entry is exactly $V_{t, c} = v_{t, c}$, representing the sequence probability in cluster $c$ when $y_1$ is chosen at step $t$.
    \item $M$ is an $m! \times m$ binary indicator matrix. Its rows represent the $m!$ possible rankings $\sigma$, and its columns represent the $m$ possible rank positions $t$. The entry $M_{\sigma, t} = 1$ if alternative $y_1$ is exactly at position $t$ in ranking $\sigma$, and $0$ otherwise. 
\end{itemize}
Since $y_1$ can occupy exactly one position in any given ranking, each row of $M$ contains exactly one $1$ and $m-1$ zeros. Multiplying row $\sigma$ of $M$ by column $c$ of $V$ simply selects the correct value $v_{t,c}$ to form $F_{\sigma, c}$. 

Since the inner dimension of both matrices is $m$, we establish the exact dimensionality bound for $F$:
\begin{align*}
\text{rank}(F) = \min \{m, 2k\}.
\end{align*}

Thus, the dimension of the null space of $F$ is larger than $1$, under the proposition's condition $m \le 2k - 1$. Therefore, the null space of $F$ is strictly non-trivial, guaranteeing the existence of a non-zero vector $\beta$ such that $F \beta = \mathbf{0}$.

Furthermore, because the rows of $F$ represent probabilities over an exhaustive sample space, $\sum_{\sigma} F_{\sigma, c} = 1$ for all $c$. Summing the rows of $F \beta = \mathbf{0}$ yields:
\begin{align*}
& \sum_{\sigma} \sum_{c=1}^{2k} F_{\sigma, c} \beta_c = 0 \\
= &  \sum_{c=1}^{2k} \beta_c \left( \sum_{\sigma} F_{\sigma, c} \right) = \sum_{c=1}^{2k} \beta_c = 0.
\end{align*}
Since $\beta$ is a non-zero vector whose elements sum to zero, it must contain both strictly positive and strictly negative elements. Let $I_+ = \{c \mid \beta_c > 0\}$ and $I_- = \{c \mid \beta_c < 0\}$. We define a normalization constant $W = \sum_{c \in I_+} \beta_c = \sum_{c \in I_-} -\beta_c$. 

We now construct the two mixtures. We define two disjoint subsets of cluster parameters $r = \{r^{(c)}\}_{c \in I_+}$ and $r' = \{r^{(c)}\}_{c \in I_-}$. We construct their corresponding mixing weights as $\alpha_c = \frac{\beta_c}{W}$ for $c \in I_+$, and $\alpha'_c = \frac{-\beta_c}{W}$ for $c \in I_-$.

Because $F \beta = \mathbf{0}$, we have:
\begin{equation*}
\sum_{c \in I_+} \alpha_c F_{\sigma, c} = \sum_{c \in I_-} \alpha'_c F_{\sigma, c} \quad \text{for all } \sigma.
\end{equation*}
This confirms that the mixture model defined by $(r, \alpha)$ and the disjoint mixture model defined by $(r', \alpha')$ produce the same marginal probability distribution over the $m$-way rankings, proving the model is non-identifiable.

\end{proof}

\subsection{Proof for Proposition \ref{prop_identifiability_with_augmented_examples}}\label{sec_proof_prop_2}

Next, we present the proof for analyzing the rank of the probability matrix of partial rankings over the augmented set of responses. 
We assume that the generated responses exhibit distinct scores. We show that, under this assumption, a probability matrix (representing all possible ranking probabilities under the mixture model) has sufficiently many independent rows to match its column dimension, thus exhibiting a full column rank.

\begin{proof}[Proof of Proposition \ref{prop_identifiability_with_augmented_examples}]
Recall that we augment the original $m$ responses with additional generated responses and expand the candidate set to $Y'$ with a size of $m'$.
In this proof, we examine the probability over the space of partial rankings over $m'$ alternatives, formed by selecting a ranked list of $m$ items and leaving the remaining $m'-m$ items unranked at the bottom. The total number of partial rankings in the expanded space is $\frac{m'!}{(m'-m)!}$.

For an augmented set of $m'$ alternatives, we denote such a partial ranking $\sigma$ as $y_{\sigma(1)} \succ y_{\sigma(2)} \succ \dots \succ y_{\sigma(m)} \succ \set{y_{\sigma(m+1)}, \dots, y_{\sigma(m')}}$. Given a score vector $r$ of dimension $m'$, the probability of observing the partial ranking $\sigma$ in a PL model is: 
\begin{align*}
P(\sigma \vert{} r ) = \prod_{i=1}^{m-1} \frac{r_{\sigma(i)}}{\sum_{j=i}^m r_{\sigma(j)} + Z(\sigma)},
\end{align*}
where $Z(\sigma) = \sum_{t=m+1}^{m'} r_{\sigma(t)}$ represents the sum of scores of alternatives ranked at the bottom. 

Given $2k$ score vectors $r^{(1)}, \dots, r^{(2k)}$, we then define the probability matrix over the augmented set of alternatives as $F^{\textup{aug}}$, which is a $\frac{m'!}{(m'-m)!} \times 2k$ matrix, where each row corresponds to one of the possible partial rankings and each column corresponds to one of the $2k$ PL components. 

In the conditions of this proposition, we assume that the newly generated augmented responses exhibit distinct scores within each PL model. For any distinct additional alternatives $i, j \in \set{m+1, \dots, m'}$, we assume $r^{(c)}_i \neq r^{(c)}_j$ for each PL model $c$. As these scores are distinct, the denominator term $Z(\sigma)$ depends on the specific combinatorial subset of items left unselected in the bottom $\sigma$. Consequently, for distinct partial rankings $\sigma_1 \neq \sigma_2$ that leave different subsets of items unranked, the denominators differ: $Z(\sigma_1) \neq Z(\sigma_2)$. 
This makes it impossible to factor $F^{\textup{aug}}$ into the rank-deficient form. 

We now show the identifiability of $F^{\textup{aug}}$ by examining its rank. We will show that for all possible pairwise different vectors, the rank of $F^{\textup{aug}}$ is equal to $2k$. We first obtain a $2k \times 2k$ matrix $\hat{F} = T \times F^{\textup{aug}}$ by linearly combining some row vectors of $F^{\textup{aug}}$ via a linear transformation $T$. Then, we show that $\operatorname{rank}(\hat{F}) = 2k$, which implies that $\operatorname{rank}(F^{\textup{aug}}) = 2k$.

Under the conditions of the proposition, we have at least $k$ alternatives in the additional augmented set ($m^\prime - m \ge k$). Without loss of generality, we select exactly $k$ items from the augmented set, denoted as $y_1, y_2, \dots, y_k$. Let $r_j^{(c)}$ denote the score of item $y_j$ in the $c$-th Plackett-Luce model, which is also the marginal probability that item $y_j$ is ranked at the top. Denote a row vector $\omega^{(j)} = [r_j^{(1)}, r_j^{(2)}, \dots, r_j^{(2k)}]$ as these marginal probabilities across all $2k$ components.

We obtain the $2k \times 2k$ matrix $\hat{F}$ by summing over the corresponding rows in $F^{\textup{aug}}$ to obtain the following probabilities. Here, the matrix $T$ is a 0-1 matrix indicating the selection of rows:
\begin{itemize}
\item The sum of all rows in $F^{\textup{aug}}$, yielding the vector $\mathbf{1} = [1, 1, \dots, 1]$.
\item The sum of all rows that rank $y_j$ at the top for $j \in \{2, \dots, k\}$, yielding the row vector $\omega^{(j)}$.
\item The sum of all rows that rank $y_j$ at the top and $y_1$ as the second, for $j \in \{2, \dots, k\}$. The $c$-th element of this row vector is $\frac{r_j^{(c)} r_1^{(c)}}{1 - r_j^{(c)}}$.
\end{itemize}
Then, the matrix $\hat{F}$ takes the following form:
\begin{align*}
    \hat{F} = \begin{bmatrix}  1 & 1 & \dots & 1 \\  r_1^{(1)} & r_1^{(2)} & \dots & r_1^{(2k)} \\  r_2^{(1)} & r_2^{(2)} & \dots & r_2^{(2k)} \\  \vdots & \vdots & \ddots & \vdots \\  r_k^{(1)} & r_k^{(2)} & \dots & r_k^{(2k)} \\  \frac{r_2^{(1)} r_1^{(1)}}{1 - r_2^{(1)}} & \frac{r_2^{(2)} r_1^{(2)}}{1 - r_2^{(2)}} & \dots & \frac{r_2^{(2k)} r_1^{(2k)}}{1 - r_2^{(2k)}} \\  \vdots & \vdots & \ddots & \vdots \\  \frac{r_k^{(1)} r_1^{(1)}}{1 - r_k^{(1)}} & \frac{r_k^{(2)} r_1^{(2)}}{1 - r_k^{(2)}} & \dots & \frac{r_k^{(2k)} r_1^{(2k)}}{1 - r_k^{(2k)}}  \end{bmatrix}.
\end{align*}
Because the parameters are non-degenerate, at least one vector in $\{\omega^{(1)}, \dots, \omega^{(k)}\}$ is linearly independent of $\mathbf{1}$. Without loss of generality, suppose $\omega^{(1)}$ is linearly independent of $\mathbf{1}$. 

Next, we prove the proposition by examining the following two cases. 

\textbf{Case 1.} $\omega^{(2)}, \dots, \omega^{(k)}$ are all linear combinations of $\mathbf{1}$ and $\omega^{(1)}$.  For all $j \in \{2, \dots, k\}$, we can rewrite $r_j^{(c)} = p_j r_1^{(c)} + q_j$ for some constants $p_j, q_j$. In this case, the entries of $\hat{F}$ are rational functions of $r_1^{(c)}$. Because the scores of the augmented items are distinct, $r_1^{(c)}$ takes on pairwise different values across the $2k$ components.

Then, we prove that the rank of $F^{\textup{aug}}$ is equal to $2k$ by contradiction. Suppose that $\operatorname{rank}(\hat{F}) < 2k$. Then, there exists a non-zero vector $t = [t_1, \dots, t_{2k}]^\top$ spanning the null space, such that $t^\top \hat{F} = \mathbf{0}$.
For the dot product of $t$ with the $c$-th column of $\hat{F}$ yields:
\begin{align*}
& t_1 + t_2 (r_1^{(c)}) + \sum_{j=2}^{k} t_{j+1} (r_j^{(c)}) + \\ & \sum_{j=2}^{k} t_{k+j} \left( \frac{r_j^{(c)} r_1^{(c)}}{1 - r_j^{(c)}} \right) = 0.
\end{align*}  
By substituting $r_j^{(c)} = p_j r_1^{(c)} + q_j$ into the above equation and letting the variable $x = r_1^{(c)}$, we can define a function $f(x)$ such that $f(x) = 0$ for every component $c$.

Then, we can convert this function into a polynomial $g(x)$ by the product of all its denominators.  The highest degree terms in $g(x)$ have a maximum possible degree of $k$. Thus, $g(x)$ is a polynomial of degree at most $k$.  Because $g(x) = 0$ must hold across $2k$ PL-models, the polynomial $x$ must have $2k$ distinct roots. However, a non-zero polynomial of degree $k$ can have at most $k$ roots. Because $k < 2k$, this creates a contradiction. Therefore, $\operatorname{rank}(\hat{F})$ must equal $2k$.

\textbf{Case 2.} There exists an $\omega^{(i)}$ (where $i \in \{2, \dots, k\}$) that is linearly independent of $\mathbf{1}$ and $\omega^{(1)}$.  Because $\omega^{(i)}$ cannot be expressed as a linear combination of $\mathbf{1}$ and $\omega^{(1)}$, inserting it into $\hat{F}$ introduces an additional orthogonal row. Following the similar logic of Case 1, we can prove that $\operatorname{rank}(\hat{F})$ equals $2k$. 

Since  the rank of $\hat{F}$ is equal to $2k$ and $\hat{F} = T F^{\textup{aug}}$, thus the rank of $F^{\textup{aug}}$ is at least $2k$. As the column dimension of $F^{\textup{aug}}$ is $2k$, thus, we have the rank of $F^{\textup{aug}}$ is $2k$. 
\end{proof}

\begin{figure*}[t!]
    \centering
    \begin{subfigure}[b]{\textwidth}
        \centering
        \begin{minipage}[b]{0.32\textwidth}
            \centering
            \includegraphics[width=\textwidth]{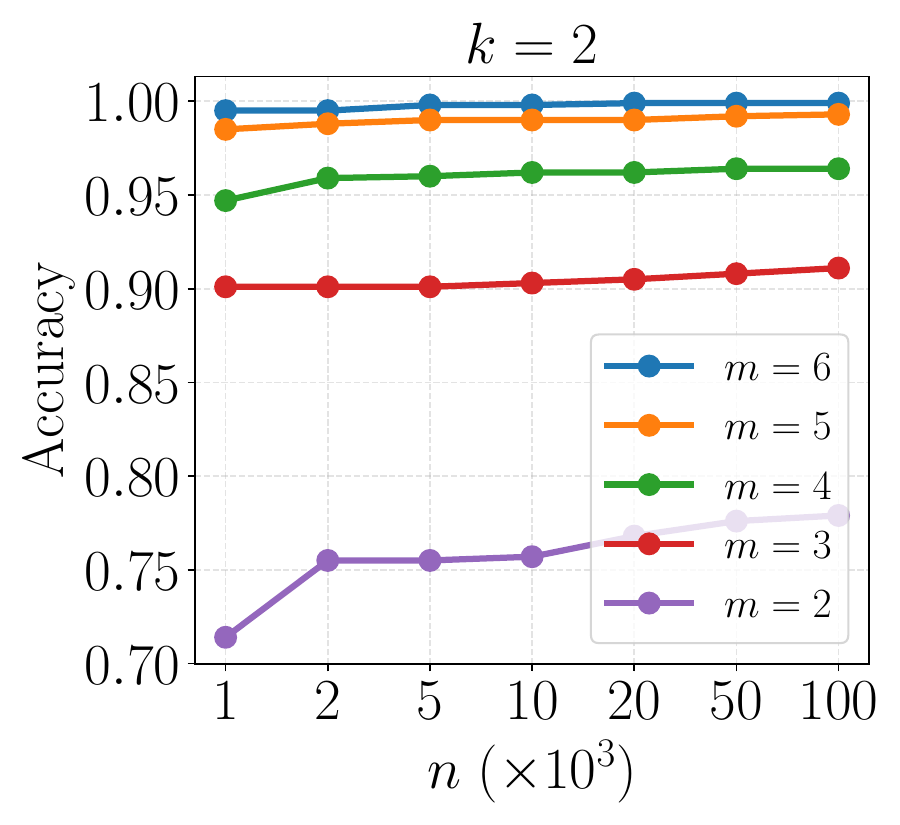}
        \end{minipage}\hfill
        \begin{minipage}[b]{0.32\textwidth}
            \centering
            \includegraphics[width=\textwidth]{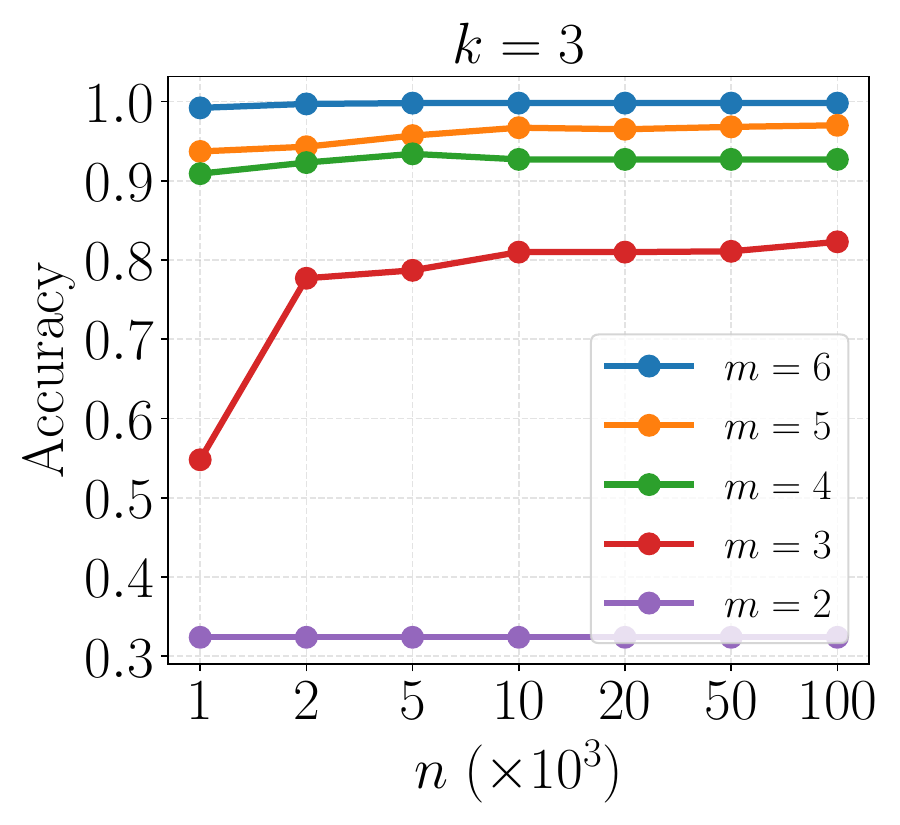}
        \end{minipage}\hfill
        \begin{minipage}[b]{0.32\textwidth}
            \centering
            \includegraphics[width=\textwidth]{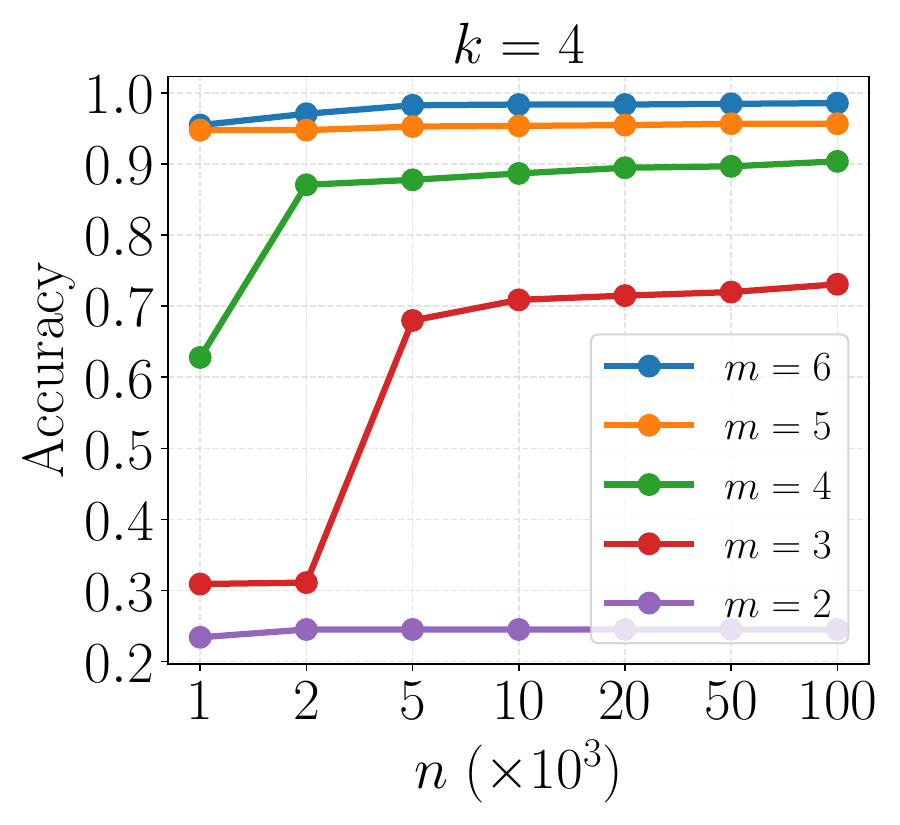}
        \end{minipage}
        \caption{{Varying the number of training rankings $n$ between $1{,}000$ and $100{,}000$. Using rankings of $m=2$ cannot recover the clusters, even for a large $n$.}}
        \label{fig_ranking_accuracy_varying_n}
    \end{subfigure}
 
    \vspace{0.15in}
 
    \begin{subfigure}[b]{\textwidth}
        \centering
        \begin{minipage}[b]{0.32\textwidth}
            \centering
            \includegraphics[width=\textwidth]{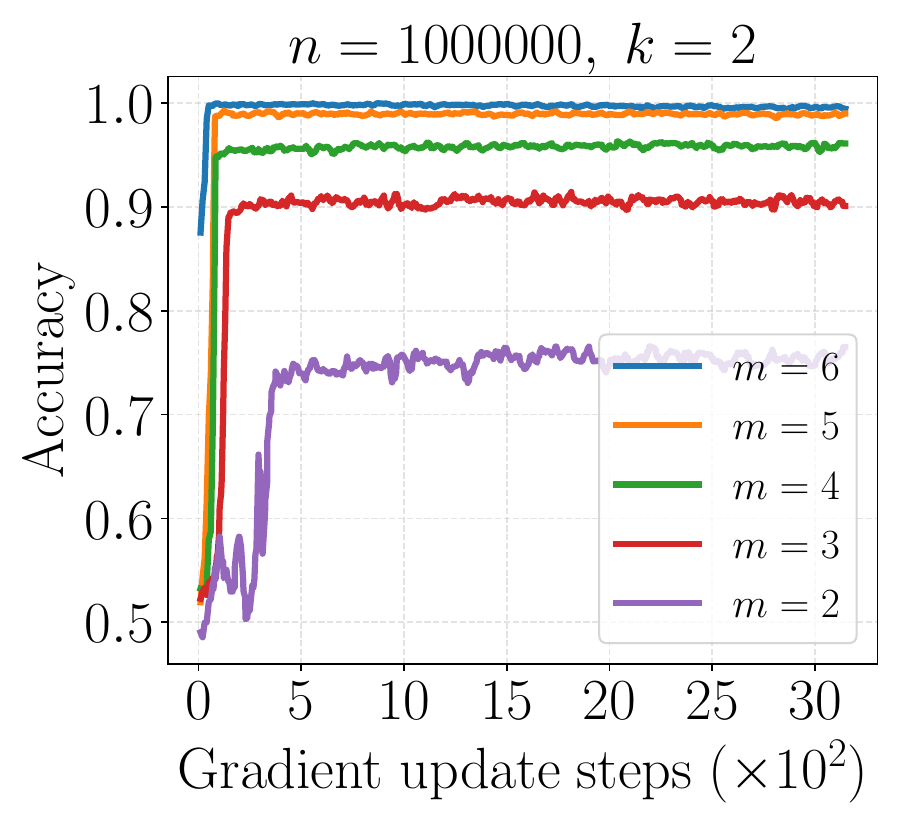}
        \end{minipage}\hfill
        \begin{minipage}[b]{0.32\textwidth}
            \centering
            \includegraphics[width=\textwidth]{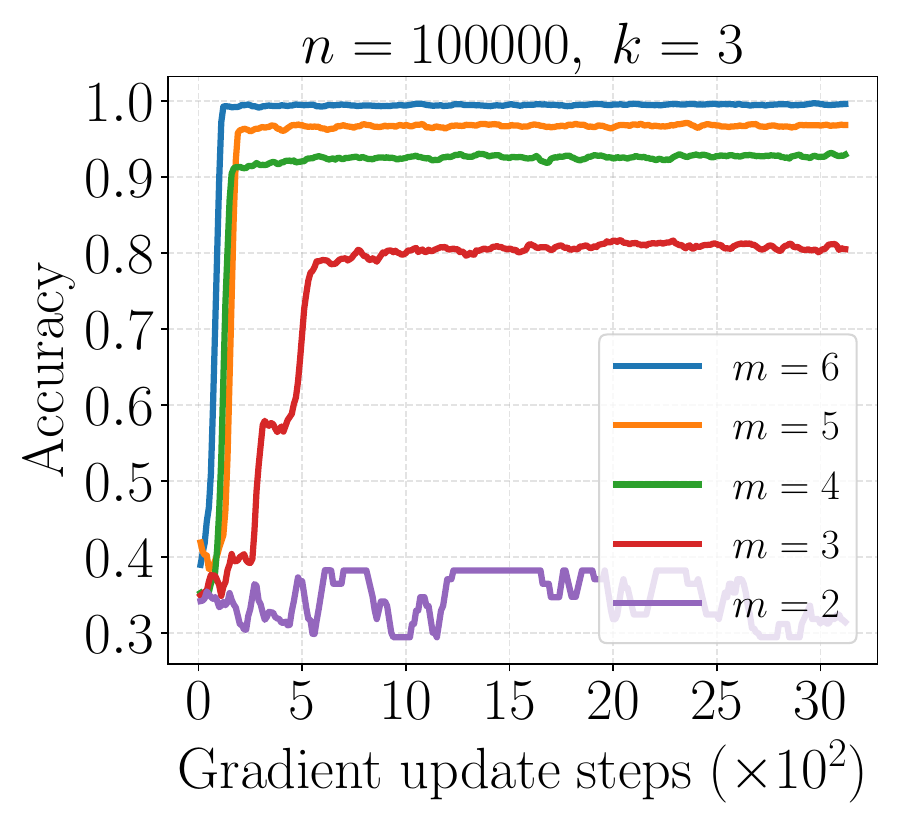}
        \end{minipage}\hfill
        \begin{minipage}[b]{0.32\textwidth}
            \centering
            \includegraphics[width=\textwidth]{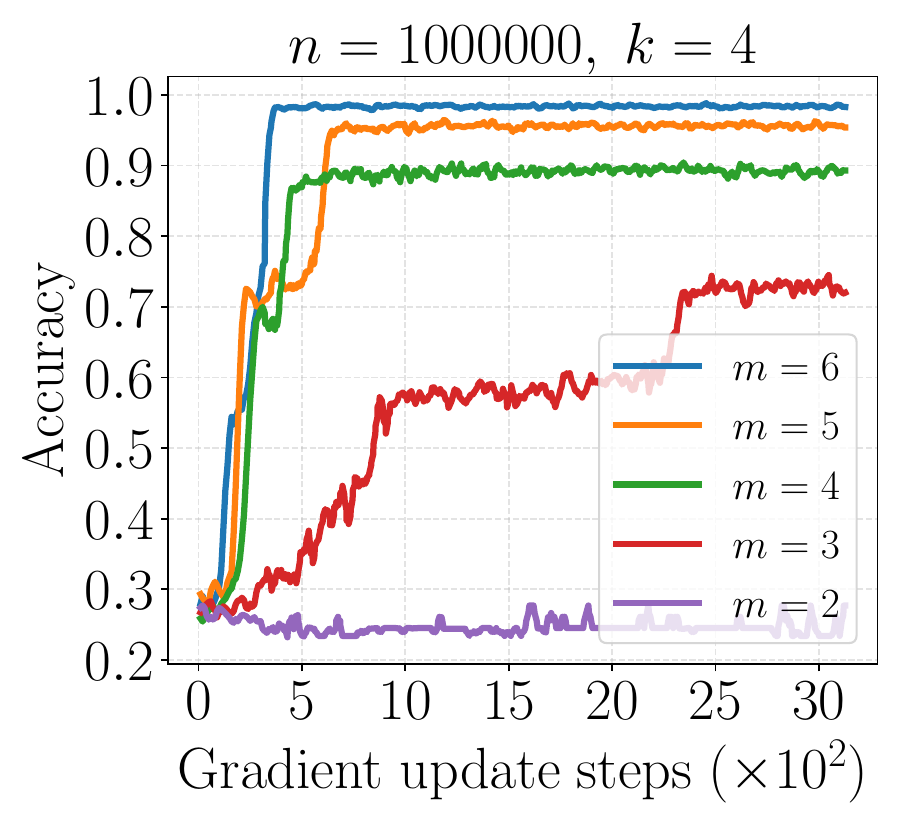}
        \end{minipage}
        \caption{{Varying the ranking size $m$ between $2, 3, 4, 5$ and $6$, at $n = 100{,}000$ rankings. Larger $m$ converges to higher accuracy.}}
        \label{fig_ranking_accuracy_varying_m}
    \end{subfigure}
    \caption{{Validation accuracy of the inferred cluster assignments on the synthetic data, for $k = 2, 3, 4$ clusters from left to right in each row. The two rows vary the two quantities that govern recovery: the amount of data in \subref{fig_ranking_accuracy_varying_n}, and the length of each ranking in \subref{fig_ranking_accuracy_varying_m}. }}
\label{fig_synthetic_combined}
\end{figure*}

\subsection{Extensions}\label{sec_extension}

We present our method using the Mallows model that defines the probability of a ranking based on its distance from an ideal ranking $\sigma_c$. Let $d_c(\cdot, \cdot)$ be a distance metric and $\phi_c(x_i) \in (0, 1)$ be a dispersion parameter dictating annotator variance. The probability of the ranking is:
\begin{equation}
\frac{\phi_c(x_i)^{d_c(\sigma_i, \sigma_c)}}{\prod_{j=1}^m \sum_{l=0}^{j-1} \phi_c(x_i)^l}.    
\end{equation}
We use the Reverse Major Index (RMJ) as the distance metric, which penalizes adjacent inversions, weighted by their rank position, where $d_c(\sigma_i, \sigma_c) 
= \sum_{t=1}^{m-1} (m - t) \cdot \mathbbm{1} \left\{ r_{\theta_c}(x_i, y_{i, \sigma_i(t)}) < r_{\theta_c}(x_i, y_{i, \sigma_i(t+1)}) \right\}$. 
The $(m-t)$ weighting imposes a larger penalty for misrankings at the top of the list. Substituting this probability into Equation \eqref{eq_mixture_of_ranking} forms the Mixture of Mallows models.

\section{Omitted Experiments}\label{sec_additional_exp}

\subsection{Synthetic data}

We now describe the details of the synthetic experiment to show the unidentifiability of the mixture of ranking models. 
We construct the dataset as follows: each response $y$ is represented as a list of $k$ integers, where $k$ is the total number of clusters. The ground-truth utility of $y$ for an annotator in cluster $c \in \{1, \dots, k\}$ is the $c$-th integer in the list. In other words, cluster $c$ ranks responses by its $c$-th integer. 
For example, consider three responses: ``[1, 2, 3]'', ``[2, 3, 1]'', ``[3, 1, 2]'', the rankings from clusters 1 and 2 are given as:
\begin{center}
    Cluster 1: [3, 1, 2] $\succ$ [2, 3, 1] $\succ$ [1, 2, 3]

    Cluster 2: [2, 3, 1] $\succ$ [1, 2, 3] $\succ$ [3, 1, 2]
\end{center}
We generate a dataset with $n$ annotators and $k$ clusters. We conduct experiments in three settings with $k$ between $2$, $3$, and $4$.   We vary the slate size $m \in \{2, 3, 4, 5, 6, 7, 8\}$. We train a Qwen-3-0.6B model using the EM algorithm and evaluate the cluster assignment accuracy on a held-out validation set.

\begin{table*}[h!]
\centering
\caption{Relative approximation error of $\epsilon_{x, y}$, tested with five language models of up to 34 billion parameters on the HelpSteer2 dataset.}\label{tab_linear_approximation_helpsteer}
\resizebox{0.8\textwidth}{!}
{\begin{tabular}{c|ccccccc}
\toprule
Distance & Qwen-0.6B & Gemma-2B & DeepSeek-7B & Llama-13B & CodeLlama-34B \\
\midrule
$0\%-5\%$   & $0.1_{\pm 0.0} \%$ & $0.1_{\pm 0.0}\%$ & $0.1_{\pm 0.0}\%$ & $0.1_{\pm 0.0}\%$ & $0.1_{\pm 0.0}\%$\\
$5\%-10\%$  & $0.7_{\pm 0.3} \%$ & $0.6_{\pm 0.3}\%$ & $0.6_{\pm 0.0}\%$ & $0.5_{\pm 0.2}\%$ & $0.5_{\pm 0.0}\%$\\
$10\%-15\%$ & $1.6_{\pm 0.1} \%$ & $1.5_{\pm 0.2}\%$ & $1.7_{\pm 0.1}\%$ & $1.6_{\pm 0.3}\%$ & $1.5_{\pm 0.4}\%$\\
$15\%-20\%$ & $3.0_{\pm 0.3} \%$ & $3.1_{\pm 0.1}\%$ & $3.4_{\pm 0.3}\%$ & $3.0_{\pm 0.5}\%$ & $3.1_{\pm 0.9}\%$\\
$20\%-25\%$ & $5.4_{\pm 0.4} \%$ & $5.4_{\pm 0.1}\%$ & $5.0_{\pm 0.3}\%$ & $5.1_{\pm 0.7}\%$ & $5.3_{\pm 0.2}\%$\\
\bottomrule
\end{tabular}}
\end{table*}

We show the validation accuracy of the cluster assignments when applying the algorithm to identify the underlying clusters with a ranking of different sizes $m$. 
In Figure \ref{fig_ranking_accuracy_varying_n}, we vary the number of rankings for training $n$ between $1,000$ and $100,000$. We found that when using ranking over $m=2$ items, the algorithm cannot fully recover the clusters, even when increasing the number of rankings $n$. For three (or four) underlying clusters, using $m=2$ achieves nearly the accuracy of random assignment (i.e., $1/k$).
In contrast, a larger value of $m$ leads to a higher accuracy. 

We further show the convergence of the algorithm in Figure \ref{fig_ranking_accuracy_varying_m}. In addition, when comparing the number of samples $n$ across settings with different numbers of clusters, we found that it requires more training samples with a large number of underlying clusters.

\begin{table*}[t!]
\centering
\caption{We report the clustering accuracies (\%) on the UltraFeedback dataset for each evaluation dimension and their average. Our approach is compared against baselines utilizing a single ranking model and mixtures of BT models. During evaluation, we infer the cluster membership of each test instance via the learned mixture model and apply the corresponding cluster-specific head to predict the final ranking. 
All results report the mean and standard deviation over three random seeds.}
\label{tab_cluster_results}
{\small
\begin{tabular}{lcccc|c}
\toprule
Clustering accuracy (\%) & Helpfulness & Honesty & Instruction following & Truthfulness & Average \\ \midrule
MiCRo & 33.4 $\pm$ 7.8 & 28.7 $\pm$ 4.5 & 22.8 $\pm$ 0.3 & 21.7 $\pm$ 5.9 & 26.6 $\pm$ 2.1 \\
MaxMin-RLHF & 27.5 $\pm$ 4.3 & 14.4 $\pm$ 2.8 & 35.3 $\pm$ 5.6 & 27.1 $\pm$ 0.5 & 27.2 $\pm$ 0.4 \\
EM-DPO & 25.1 $\pm$ 1.9 & 22.7 $\pm$ 2.7 & 33.2 $\pm$ 5.4 & 23.2 $\pm$ 2.9 & 26.9 $\pm$ 2.1 \\
\midrule
\acronym{} (Alg. \ref{alg_em_mixture_pl})  & 68.4 $\pm$ 2.3 & 73.8 $\pm$ 0.4 & 68.4 $\pm$ 1.3 & 72.9 $\pm$ 3.1 & 70.9 $\pm$ 2.3 \\ 
\bottomrule
\end{tabular}}
\end{table*}

\subsection{Evaluation of approximation errors}

Table \ref{tab_linear_approximation_helpsteer} shows the relative approximation error of the gradient estimation method across five different language models when tested on the HelpSteer2 dataset. The models evaluated range from 0.6 billion up to 34 billion parameters, specifically including Qwen-0.6B, Gemma-2B, DeepSeek-7B, Llama-13B, and CodeLlama-34B. The error remains low across all model sizes, starting at 0.1\% for the closest distances and remaining between 5.0\% and 5.4\% for the largest distance bucket.

\begin{table*}[t!]
\centering
\caption{We report the clustering and ranking accuracies (\%) on the Persona dataset across each cluster. Our approach is compared against baselines utilizing a single ranking model and mixtures of BT models. During evaluation, we infer the cluster membership of each test instance via the learned mixture model and apply the corresponding cluster-specific head to predict the final ranking. Ranking accuracy is measured as the proportion of correctly predicted pairs over all pairs. All results report the mean and standard deviation over three random seeds.}
\label{tab_cluster_results_persona}
{\small
\begin{tabular}{lcccccc}
\toprule
 & Persona 1 & Persona 2 & Persona 3 & Persona 4 & Persona 5 & Persona 6 \\ 
\midrule
{Clustering accuracy (\%)}\\
\midrule
MiCRo & 11.7 $\pm$ 0.6 & 10.5 $\pm$ 0.3 & 21.7 $\pm$ 0.4 & 16.6 $\pm$ 1.7 & 15.0 $\pm$ 0.0  & 47.8 $\pm$ 1.3 \\
MaxMin-RLHF & 35.2 $\pm$ 1.9 & 15.7 $\pm$ 0.9 & 8.0 $\pm$ 0.7 & 8.3 $\pm$ 0.3 & 2.0 $\pm$ 0.0  & 30.4 $\pm$ 1.3 \\
EM-DPO & 58.8 $\pm$ 0.2 & 26.3 $\pm$ 1.2 & 8.7 $\pm$ 0.0 & 8.3 $\pm$ 0.3 & 25.0 $\pm$ 0.0 & 39.1 $\pm$ 1.3 \\
\acronym{} (Alg. \ref{alg_em_mixture_pl}) & 61.1 $\pm$ 1.1 & 60.0 $\pm$ 0.4 & 58.3 $\pm$ 0.3 & 50.0 $\pm$ 0.0 & 60.0 $\pm$ 0.5 & 58.3 $\pm$ 1.3 \\ 
\midrule
{Ranking accuracy (\%)} \\ 
\midrule 
DPO  & 52.9 $\pm$ 2.4 & 53.7 $\pm$ 0.8 & 53.3 $\pm$ 0.6 & 57.1 $\pm$ 0.4 & 54.2 $\pm$ 3.0 & 53.3 $\pm$ 1.6 \\
LiPO & 44.4 $\pm$ 0.4 & 60.0 $\pm$ 0.0 & 54.1 $\pm$ 0.6 & 58.3 $\pm$ 1.3 & 65.0 $\pm$ 0.1 & 58.3 $\pm$ 0.3 \\
MiCRo & 64.7 $\pm$ 1.1 & 52.6 $\pm$ 0.3 & 65.2 $\pm$ 1.2 & 50.0 $\pm$ 0.0 & 30.0 $\pm$ 0.0 & 56.5 $\pm$ 1.2 \\
MaxMin-RLHF & 52.9 $\pm$ 0.4 & 63.1 $\pm$ 1.6 & 56.5 $\pm$ 1.2 & 58.3 $\pm$ 1.3 & 35.0 $\pm$ 1.0 & 56.5 $\pm$ 2.0 \\
EM-DPO & 70.5 $\pm$ 0.9 & 52.6 $\pm$ 1.3 & 39.1 $\pm$ 2.3 & 58.3 $\pm$ 2.3 & 40.0 $\pm$ 2.0 & 56.5 $\pm$ 2.0 \\
\acronym{} (Alg. \ref{alg_em_mixture_pl})  & 83.3 $\pm$ 1.3 & 83.3 $\pm$ 2.3 & 80.0 $\pm$ 2.0 & 58.3 $\pm$ 0.3 & 80.0 $\pm$ 1.0 & 70.8 $\pm$ 0.3 \\
\midrule
\midrule
 & Persona 7 & Persona 8 & Persona 9 & Persona 10 & Persona 11 & Persona 12 \\ 
\midrule
{Clustering accuracy (\%)} \\
\midrule
MiCRo & 33.3 $\pm$ 0.2 & 14.2 $\pm$ 0.9 & 8.3 $\pm$ 0.3 & 25.0 $\pm$ 0.0 & 13.6 $\pm$ 0.4 & 16.6 $\pm$ 1.7 \\
MaxMin-RLHF & 11.1 $\pm$ 0.1 & 9.5 $\pm$ 0.2 & 20.8 $\pm$ 0.3 & 10.0 $\pm$ 0.0 & 22.7 $\pm$ 0.3 & 33.3 $\pm$ 1.3 \\
EM-DPO & 14.2 $\pm$ 0.6 & 2.4 $\pm$ 0.4 & 29.1 $\pm$ 0.7 & 20.0 $\pm$ 0.0 & 13.6 $\pm$ 0.4 & 27.7 $\pm$ 1.8 \\
\acronym{} (Alg. \ref{alg_em_mixture_pl}) & 50.0 $\pm$ 0.0 & 54.5 $\pm$ 0.4 & 57.6 $\pm$ 1.9 & 63.6 $\pm$ 0.3 & 66.6 $\pm$ 1.7 & 61.1 $\pm$ 1.1 \\ 
\midrule
{Ranking accuracy (\%)} \\ 
\midrule 
DPO  & 53.3 $\pm$ 3.6 & 52.9 $\pm$ 0.4 & 54.6 $\pm$ 1.2 & 54.6 $\pm$ 0.2 & 56.3 $\pm$ 2.0 & 57.5 $\pm$ 0.6 \\
LiPO & 55.5 $\pm$ 0.5 & 50.0 $\pm$ 0.7 & 42.3 $\pm$ 0.0 & 45.4 $\pm$ 0.5 & 66.6 $\pm$ 1.6 & 77.7 $\pm$ 0.7 \\
MiCRo & 38.8 $\pm$ 0.9 & 71.4 $\pm$ 1.3 & 62.5 $\pm$ 2.0 & 35.0 $\pm$ 0.6 & 45.4 $\pm$ 0.5 & 50.0 $\pm$ 2.0 \\
MaxMin-RLHF & 50.0 $\pm$ 1.0 & 52.3 $\pm$ 1.8 & 58.3 $\pm$ 0.3 & 40.0 $\pm$ 2.0 & 36.3 $\pm$ 1.6 & 38.8 $\pm$ 0.9 \\
EM-DPO & 61.1 $\pm$ 1.1 & 42.8 $\pm$ 1.6 & 45.8 $\pm$ 1.3 & 35.0 $\pm$ 0.5 & 36.3 $\pm$ 1.6 & 50.0 $\pm$ 2.0 \\
\acronym{} (Alg. \ref{alg_em_mixture_pl})  & 100.0 $\pm$ 0.0 & 63.6 $\pm$ 1.3 & 80.7 $\pm$ 1.6 & 63.6 $\pm$ 1.3 & 79.1 $\pm$ 1.6 & 66.6 $\pm$ 1.6 \\
\bottomrule
\end{tabular}}
\end{table*}

\subsection{Comparing fine-tuned model generation}

We evaluate the generation of our fine-tuned models on the UltraFeedback dataset. We compute a win rate against the given response ranked at the top, given a reference response in the dataset. For each test prompt, we generate responses from both models and score them using a reward model trained solely on the ground-truth rankings of that corresponding cluster. The win rate is the percentage of times the fine-tuned model's response achieves a higher reward score than the baseline's response.

We find that our method achieves an average win rate of 58\%, outperforming baselines using single ranking models by 7\% and mixture-of-BT baselines by 4.3\% across all evaluated clusters. This indicates that our method aligns the language model's generations with the distinct ranking distributions of each cluster. In contrast, because pairwise mixture models struggle to accurately identify the clusters, they fail to tailor their generations to the specific target cluster. 

\subsection{Additional computation and parameters}

Our algorithm generates additional responses prior to training, accounting for only a small fraction of the total runtime. Evaluated on a single A6000 GPU, response generation takes 1.1 GPU hours for the UltraFeedback dataset and 1.3 GPU hours for the Persona dataset, whereas model training requires 7.3 and 7.2 GPU hours, respectively. Overall, response generation consumes roughly 16\% of the total runtime. 

We use LoRA adapters to implement the mixture of PL models. Each PL model corresponds to a distinct set of LoRA adapters. Thus, a mixture of $k$ models requires $k$ sets of adapters on top of the base model. In our experiments, setting $k=4$ for the UltraFeedback dataset introduces 13.5\% additional parameters relative to the base model. For the Persona dataset, setting $k=12$ results in 40.5\% additional parameters.

\begin{table*}[t!]
\centering
\caption{We report the ablation study results of varying the sampling temperature during response generation on the UltraFeedback dataset. We compare the ranking accuracy across four criteria and their average across temperatures of 0.5, 1.0, 2.0, and 4.0. Ranking accuracy is measured as the proportion of correctly predicted pairs out of the six possible pairs per four-response ranking. All results report the mean and standard deviation over three random seeds.}
\label{tab_ablation_temperature}
{\small
\begin{tabular}{lccccc}
\toprule
Sampling Temperature & {Helpfulness} & {Honesty} & {Instruction-following} & {Truthfulness} & {Average} \\
\midrule
0.5 & $76.6 \pm 2.5$ & $69.7 \pm 1.6$ & $71.5 \pm 1.7$ & $75.8 \pm 1.8$ & $70.9 \pm 1.9$ \\
1.0 & $77.2 \pm 1.8$ & $72.8 \pm 2.9$ & $73.2 \pm 2.5$ & $68.9 \pm 1.3$ & $72.8 \pm 2.1$ \\
2.0 & $79.4 \pm 3.9$ & $73.0 \pm 1.3$ & $76.9 \pm 2.6$ & $70.5 \pm 1.3$ & $75.0 \pm 2.9$ \\
4.0 & $79.5 \pm 1.9$ & $72.9 \pm 2.1$ & $76.0 \pm 2.7$ & $71.7 \pm 1.8$ & $74.7 \pm 2.3$ \\
\bottomrule
\end{tabular}
}
\end{table*}

\begin{table*}[t!]
\centering
\caption{We report the average ranking accuracy on the Persona dataset across 12 demographic profiles when varying key hyperparameters: the number of clusters ($k$), the augmented ranking size ($m'$), and the number of anchors ($a$). We vary one parameter at a time while keeping the others fixed at their default values ($k=12, m'=26, a=6$). All results report the mean and standard deviation over three random seeds.}
\label{tab_hyperparameter_ablation}
\small
\begin{tabular}{lcccc}
\toprule
\multicolumn{5}{l}{{Varying Number of Clusters} (Fixed: $m'=26, a=6$)} \\
\midrule
Parameter & $k=4$ & $k=8$ & $k=12$ & $k=16$ \\
Accuracy & $65.2 \pm 1.6$ & $69.6 \pm 0.9$ & $76.4 \pm 0.6$ & $76.7 \pm 0.4$ \\
\midrule
\multicolumn{5}{l}{{Varying Augmented Ranking Size} (Fixed: $k=12, a=6$)} \\
\midrule
Parameter & $m'=12$ & $m'=20$ & $m'=26$ & $m'=32$ \\
Accuracy & $59.6 \pm 2.6$ & $65.1 \pm 1.2$ & $76.4 \pm 0.6$ & $76.8 \pm 0.8$ \\
\midrule
\multicolumn{5}{l}{{Varying Number of Anchors} (Fixed: $k=12, m'=26$)} \\
\midrule
Parameter & $a=2$ & $a=4$ & $a=6$ & $a=8$ \\
Accuracy & $63.8 \pm 1.6$ & $72.4 \pm 1.8$ & $76.4 \pm 0.6$ & $76.4 \pm 0.9$ \\
\bottomrule
\end{tabular}
\end{table*}

\begin{table*}[t!]
\centering
\begin{minipage}[t]{0.4\textwidth}
\centering
\caption{{Datasets used in our experiments. Recall that $m$ is the number of options and $k$ is the number of groups.}}
\label{tab_datasets_and_models}
{\small
\begin{tabular}{lccl}
\toprule
{Dataset} & $m$ & $k$ & Sources \\
\midrule
UltraFeedback & 4 & 4 & \href{https://huggingface.co/datasets/openbmb/UltraFeedback}{{openbmb/UltraFeedback}} \\
Persona & 2 & 12 & \href{https://huggingface.co/datasets/SynthLabsAI/PERSONA}{{SynthLabsAI/PERSONA}}\\
HelpSteer2 & 2 & 5 & \href{https://huggingface.co/datasets/nvidia/HelpSteer2}{{nvidia/HelpSteer2}}\\
\bottomrule
\end{tabular}}
\end{minipage}\hfill
\begin{minipage}[t]{0.5\textwidth}
\centering
\caption{{Base language models used in our experiments.}}
\label{tab_models_only}
{\small
\begin{tabular}{lll}
\toprule
{Models}  & {Sources} \\
\midrule
Qwen-3-0.6B & \href{https://huggingface.co/Qwen/Qwen3-0.6B}{{\small Qwen/Qwen3-0.6B}} \\
Gemma-2B & \href{https://huggingface.co/google/gemma-2-2b-it}{{\small google/gemma-2-2b-it}} \\
DeepSeek-7B & \href{https://huggingface.co/deepseek-ai/deepseek-llm-7b-base}{{\small deepseek-ai/deepseek-llm-7b-base}} \\
Llama-13B  & \href{https://huggingface.co/meta-llama/Llama-2-13b-chat-hf}{{\small meta-llama/Llama-2-13b-chat-hf}} \\
CodeLlama-34B & \href{https://huggingface.co/meta-llama/CodeLlama-34b-Instruct-hf}{{\small meta-llama/CodeLlama-34b-Instruct-hf}} \\
\bottomrule
\end{tabular}}
\end{minipage}
\end{table*}

\subsection{Complete comparison}

We provide the complete results corresponding to the experiments and ablation studies discussed in the main text. Table \ref{tab_cluster_results} presents the clustering and ranking accuracies broken down by individual evaluation dimensions and their averages.

Table \ref{tab_cluster_results_persona} reports the clustering and ranking accuracies across all twelve demographic profiles on the Persona dataset.

Table \ref{tab_ablation_temperature} reports the impact of varying the response generation sampling temperature on ranking accuracy across the evaluation criteria.

Table \ref{tab_hyperparameter_ablation} reports the results of our hyperparameter ablation study of varying the number of clusters, the augmented ranking size, and the number of anchors.

\subsection{Discussions}

\textbf{Datasets, model sizes, and compute cost.} We evaluate our approach across several datasets that exhibit multi-dimensional evaluations and diverse user demographics.  The details of the datasets, including their categories, ranking properties, and sources, are summarized in Table \ref{tab_datasets_and_models}.

We conduct experiments using a range of base language models with varying parameter sizes to validate our gradient estimation approach. The models and their parameter counts are detailed in Table \ref{tab_models_only}.

The total computational cost is described in Section \ref{sec_experiments} of this paper. We evaluate the runtime and memory cost on a machine with a single RTX-A6000 GPU and 48 CPUs.

{}

\end{document}